\documentclass[11pt]{article}
\usepackage[includeheadfoot,margin=1in]{geometry}
\usepackage{times}
\usepackage{amsmath,amsfonts,amssymb}
\usepackage{mathrsfs}
\usepackage{mathtools}
\usepackage{enumerate}
\usepackage{verbatim}
\usepackage{commath}
\usepackage[usenames]{color}
\definecolor{plum}  {rgb}{.4,0,.4}
\definecolor{BrickRed} {rgb}{0.6,0,0}
\definecolor{DarkBlue} {rgb}{0,0,0.6}
\usepackage[plainpages=false,pdfpagelabels,colorlinks=true,linkcolor=BrickRed,citecolor=plum,urlcolor=DarkBlue]{hyperref}
\usepackage[round]{natbib}
\usepackage[mathcal]{euscript}
\usepackage{cleveref}

\def\ddefloop#1{\ifx\ddefloop#1\else\ddef{#1}\expandafter\ddefloop\fi}

\def\ddef#1{\expandafter\def\csname b#1\endcsname{\ensuremath{\boldsymbol{#1}}}}
\ddefloop ABCDEFGHIJKLMNOPQRSTUVWXYZabcdefghijklmnopqrstuvwxyz\ddefloop

\def\ddef#1{\expandafter\def\csname c#1\endcsname{\ensuremath{\mathcal{#1}}}}
\ddefloop ABCDEFGHIJKLMNOPQRSTUVWXYZ\ddefloop
 
\def\ddef#1{\expandafter\def\csname s#1\endcsname{\ensuremath{\mathsf{#1}}}}
\ddefloop ABCDEFGHIJKLMNOPQRSTUVWXYZ\ddefloop

\def\Reals{{\mathbb R}}

\def\E{{\mathbf E}} 
\def\P{{\mathbf P}} 
\def\eps{\varepsilon}
\def\bd#1{{\boldsymbol{#1}}}

\makeatletter
\newsavebox{\@brx}
\newcommand{\llangle}[1][]{\savebox{\@brx}{\(\m@th{#1\langle}\)}%
  \mathopen{\copy\@brx\kern-0.5\wd\@brx\usebox{\@brx}}}
\newcommand{\rrangle}[1][]{\savebox{\@brx}{\(\m@th{#1\rangle}\)}%
  \mathclose{\copy\@brx\kern-0.5\wd\@brx\usebox{\@brx}}}
\makeatother

\newenvironment{proof}{\paragraph*{Proof:}}{\hfill$\square$}
\newtheorem{theorem}{Theorem}
\newtheorem{definition}{Definition}
\newtheorem{assumption}{Assumption}

\newtheorem{proposition}{Proposition}
\newtheorem{lemma}{Lemma}

\def\act{\sigma} 
\def\FE{F} 
\def\Wass{\cW}
\def\fE{\mathscr{E}}
\def\fM{\mathscr{M}}
\def\fP{\mathscr{P}}
\def\fU{\mathscr{U}}
\def\deq{:=}
\def\1{{\mathbf{1}}}
\def\ave#1{\langle #1 \rangle}
\def\Law{{\mathrm{Law}}}

\begin{document}

\title{Function approximation by neural nets in the mean-field regime: Entropic regularization and controlled McKean--Vlasov dynamics}
\author{Belinda Tzen \\ \href{mailto:bt2314@columbia.edu}{bt2314@columbia.edu} \\ 
\and Maxim Raginsky \\ \href{mailto:maxim@illinois.edu}{maxim@illinois.edu} }
\date{}
\maketitle

\begin{abstract}
We consider the problem of function approximation by two-layer neural nets with random weights that are ``nearly Gaussian'' in the sense of Kullback-Leibler divergence. Our setting is the mean-field limit, where the finite population of neurons in the hidden layer is replaced by a continuous ensemble. We show that the problem can be phrased as global minimization of a free energy functional on the space of (finite-length) paths over probability measures on the weights. This functional trades off the $L^2$ approximation risk of the terminal measure against the KL divergence of the path with respect to an isotropic Brownian motion prior. We characterize the unique global minimizer and examine the dynamics in the space of probability measures over weights that can achieve it. In particular, we show that the optimal path-space measure corresponds to the F\"ollmer drift, the solution to a McKean-Vlasov optimal control problem closely related to the classic Schr\"odinger bridge problem. While the F\"ollmer drift cannot in general be obtained in closed form, thus limiting its potential algorithmic utility, we illustrate the viability of the mean-field Langevin diffusion as a finite-time approximation under various conditions on entropic regularization. Specifically, we show that it closely tracks the F\"ollmer drift when the regularization is such that the minimizing density is log-concave.\end{abstract}


\section{Introduction}

Beginning with the seminal work of \citet{barron1992,barron1993}, there has been a lot of interest in efficient neural net approximation of functions that can be represented in the form
\begin{align}\label{eq:meanfield}
	f(x) = \int_{\Reals \times \Reals^d} \alpha\act(x;w)\mu(\dif \alpha,\dif w),
\end{align}
where $x$ is the input, $(\alpha,w)$ is the finite-dimensional vector of parameters taking values in $\Reals \times \Reals^d$, $\act(x;w)$ is a suitably chosen nonlinearity, and $\mu$ is a probability measure on the parameter space.  To pass from the continual representation \eqref{eq:meanfield} to a finite net, one could draw a large number $N$ of independent samples $(\alpha^1,w^1),\ldots,(\alpha^N,w^N)$ from $\mu$ and approximate \eqref{eq:meanfield} with
\begin{align}\label{eq:finite_N}
	\hat{f}_N(x) = \frac{1}{N}\sum^N_{i=1}\alpha^i\act(x;w^i).
\end{align}
Concentration-of-measure techniques can be used to show that the finite-$N$ neural net \eqref{eq:finite_N} sharply concentrates around its idealized counterpart \eqref{eq:meanfield}; in addition to the works of Barron mentioned above, see \cite{yukich1995,donahue1997,gurvits1997}. This perspective of conceptualizing neural networks as continuous systems has informed some recent theoretical advances, such as the mean-field theory of two-layer neural nets \citep{chizat_bach18,mei2018meanfield,mei2019meanfield,sirignano2020meanfield_CLT,sirignano2020meanfield_LLN,rotskoff2018meanfield} or the Gaussian universality theory for randomly initialized deep and wide neural nets \citep{neal1996,lee2018gp,lee2019wide,eldan21}. In these works, one starts with the finite object \eqref{eq:finite_N} and then analyzes the so-called mean-field limit as $N \to \infty$.

The mean-field perspective allows us to reason about the evolution of weights during training given i.i.d.\ data $(X_i,Y_i)$ with $f(x) = \E[Y_i|X_i = x]$ as a gradient flow in the space of probability measures over the weights, as this so-called \textit{distributional dynamics} can be captured by a certain nonlinear PDE of McKean--Vlasov type \citep{kolokoltsov2010nonlinear}. The mean-field description preserves the essential features of the optimization landscape that are insensitive to the number of neurons. A finite-size network in this case can be conceived of as providing an empirical distribution over the weights, and, under mild regularity assumptions on the data and on the activation function, we can transfer results from the continuous-time, infinite-width setting to the discrete-time, finite-width setting.

In this paper, we are interested not in the dynamics of neural net weights during training, but rather in the function approximation problem: Given a function $f$, can we find a probability measure $\mu$ of low complexity (in a sense to be made precise shortly), such that the integral on the right-hand side of \eqref{eq:meanfield} provides a good approximation to $f$ (say, in $L^2$ sense)? Our analysis of this problem consists of two complementary parts:  a static formulation, where we seek to minimize a certain free-energy objective; and a dynamic formulation, which characterizes the evolution towards the optimal distribution as a \textit{Brownian transport map} \citep{Mikulincer2024}, i.e., a measurable map that sends a realization of the standard Brownian motion to a sample from the optimal distribution subject to causality constraints.

For the first (static) part, the free energy functional is a linear combination of the $L^2$ risk of a distribution and its Kullback--Leibler divergence from an isotropic Gaussian. The free energy, which has appeared in a different guise in the work of \citet{mei2018meanfield}, is parametrized by the variance $T$ of this Gaussian prior and by the regularization (inverse temperature) parameter $\beta$ that controls the trade-off between these two terms. We show that, under appropriate regularity conditions, the free energy has a unique minimizer and provide explicit upper bounds on both the KL divergence and the $L^2$ Wasserstein distance between this minimizer and the Gaussian prior.

The second (dynamic) part entails setting up a finite-time optimal stochastic control problem for a stochastic differential equation (SDE) on the space of weights.  The control law that achieves minimum quadratic running cost subject to a desired terminal density corresponds to the F\"ollmer drift \citep{follmer1985reversal,daipra1991reciprocal,lehec2013entropy,eldan2018diffusion} that also solves the entropic optimal transport problem for the optimal measure. Working with the F\"ollmer drift here acquires the additional complication that it depends on the law of the process that minimizes the free energy; due to this nonlinear dependence of the terminal cost on the target probability law, we have to employ the machinery of \textit{controlled McKean--Vlasov dynamics} \citep{carmona2015MKV}. The resulting distributional dynamics consists of two coupled PDEs, the forward (Fokker--Planck) equation that governs the evolution of the probability density of the weights and the backward (Hamilton--Jacobi--Bellman) equation that governs the evolution of the value function of the McKean--Vlasov control problem. A key structural property of the optimal distribution that arises from the static analysis, the so-called \textit{Boltzmann fixed-point property}, turns out to be crucial for resolving this coupling between the forward and the backward dynamics.

\section{Problem setup}

We consider the problem of approximating a target function $f : \cX \to \Reals$ by a two-layer neural net with $N$ hidden-layer neurons. Here, $\cX$ is a Borel subset of $\Reals^p$, and the neural nets will take the form
\begin{align*}
	\hat{f}_N(x; \bd{w}) = \frac{1}{N}\sum^N_{i=1} \act(x; w^i),
\end{align*}
where $x \in \Reals^p$ is the input (or feature) vector, $\bd{w} = (w^1,\ldots,w^N)$ is the $N$-tuple of weights $w^i \in \Reals^d$, and $\act : \Reals^p \times \Reals^d \to \Reals$ is an activation function. This form of the net is simpler than the representations in \eqref{eq:meanfield} and \eqref{eq:finite_N}: the coefficients $\alpha^1,\dots,\alpha^N$ are all fixed and set to $1$.

We will measure the accuracy of approximation using $L^2(\pi)$ risk, where $\pi$ is a fixed Borel probability measure supported on $\cX$:
\begin{align}\label{eq:L2risk}
	R_N(\bd{w}) \deq \| f - \hat{f}_N(\cdot; \bd{w}) \|^2_{L^2(\pi)} = \int_\cX \pi(\dif x) |f(x)-\hat{f}_N(x;\bd{w})|^2.
\end{align}
It is expedient to express the risk \eqref{eq:L2risk} as
\begin{align}\label{eq:L2risk_alt}
	R_N(\bd{w}) = R_0 + \frac{2}{N}\sum^N_{i=1} \tilde{f}(w^i) + \frac{1}{N^2}\sum^N_{i=1}\sum^N_{j=1} K(w^i,w^j),
\end{align}
where $R_0 \deq \E_\pi[|f(X)|^2]$, $\tilde{f}(w) \deq - \E_\pi[f(X)\act(X;w)]$, and $K(w,w') \deq \E_\pi[\act(X;w)\act(X;w')]$. The alternative form \eqref{eq:L2risk_alt} of the $L^2$ risk makes it apparent that it depends only on the empirical distribution of the weights (and, in particular, is invariant under permutations of the neurons). Moreover, if we define the mapping $\hat{f} : \cX \times \fP(\Reals^d) \to \Reals$ by
\begin{align*}
	\hat{f}(x; \mu) \deq \int_{\Reals^d} \act(x; w)\mu(\dif w),
\end{align*}
and the associated $L^2(\pi)$ risk
\begin{align}\label{eq:L2risk_mf}
	R(\mu) \deq \| f - \hat{f}(\cdot;\mu)\|^2_{L^2(\pi)} = R_0 + 2\int_{\Reals^d} \tilde{f}\dif \mu + \int_{\Reals^d \times \Reals^d} K \dif\,(\mu \otimes \mu),
\end{align}
then $\hat{f}_N(x;\bd{w}) = \hat{f}(x;\hat{\mu}_{\bd{w}})$ and $R_N(\bd{w}) = R(\hat{\mu}_{\bd{w}})$, where $\hat{\mu}_{\bd{w}} \deq \frac{1}{N}\sum^N_{i=1} \delta_{w^i}$ is the empirical distribution of $\bd{w}$. This lifting from finite populations of neurons to continual ensembles is the essence of the mean-field theory of neural nets. Our focus in this work will be on managing the trade-off between the risk $R(\mu)$ and the relative entropy (or Kullback–Leibler divergence) between $\mu$ and an isotropic Gaussian prior.

Specifically, we will show that we can introduce a stochastic dynamics in the space of weights that leads to the probability law $\mu^\star$ achieving the optimal trade-off between risk and relative entropy. For now, we keep the discussion informal. First, let us give a heuristic motivation for a finite net with $N$ hidden units. Let $\bd{B}^i = (B^i_t)_{t \ge 0}$, $i \in [N]$, be $N$ independent copies of a standard $d$-dimensional Brownian motion. Consider the following system of $N$ coupled It\^o SDEs:
\begin{align}\label{eq:particles}
	\dif W^i_t = u^i_t \dif t + \dif B^i_t, \qquad W^i_0 \equiv 0, i \in [N];\, t \in [0,T]
\end{align}
for some finite time horizion $T > 0$, where $\bd{u}^i = \{u^i_t\}_{t \ge 0}$ are $N$ progressively measurable $\Reals^d$-valued processes.  The dynamics in \eqref{eq:particles} governs the evolution of the weights of $N$ hidden neurons, where each neuron adds a drift $\bd{u}^i$ to its  Brownian motion $\bd{B}^i$. Given a regularization parameter $\beta > 0$, let us define the expected cost for the $i$th neuron by
\begin{align}\label{eq:particle_cost}
	J^{N,i}_{\beta,T}(\bd{u}^1,\ldots,\bd{u}^N) \deq \E\left[\frac{1}{2}\int^T_0 \|u^i_t\|^2 \dif t + \frac{\beta}{2}R_N(\bd{W}_T)\right],
\end{align}
where $\bd{W}_T \deq (W^1_T,\ldots,W^N_T)$ are the weights of the $N$ neurons at time $t=T$ generated according to \eqref{eq:particles} and $R_N(\cdot)$ is the $L^2$ risk as defined in \eqref{eq:L2risk} and \eqref{eq:L2risk_alt}. Each choice of the drift processes $\bd{u}^1,\ldots,\bd{u}^N$ corresponds to a \textit{strategy} for updating the weights of the $N$ hidden units. Now, if each $\bd{u}^i$ is identically zero, then $\bd{W}_1 = (W^1_T,\dots,W^N_T)$ will be an $N$-tuple of i.i.d.\ copies of a zero-mean $d$-dimensional Gaussian random vector with covariance matrix $TI_d$. Thus, the time horizon $T$ controls the variance of the isotropic Gaussian prior, and each  nonzero strategy $(\bd{u}^1,\dots,\bd{u}^N)$ will steer the probability distribution of $\bd{W}_1$ away from the Gaussian prior. Hence, the first term on the right-hand side of \eqref{eq:particle_cost} is a regularization term that penalizes large excursions from the prior, and the regularization parameter $\beta$ controls the relative importance of the risk term and the regularization term.

Since $R_N(\cdot)$ is symmetric w.r.t.\ permutations of the $N$ neurons, we may limit ourselves to \textit{exchangeable} strategies, i.e., when the joint probability law of $(\bd{u}^i,\bd{B}^i)_{1 \le i \le N}$ is invariant under permutations of the neurons. For any such strategy, the costs in \eqref{eq:particle_cost} are all equal, and it is natural to pose the problem of minimizing \eqref{eq:particle_cost} over all exchangeable strategies. This minimization is generally intractable; to simplify the analysis, we can  pass to the mean-field limit by taking $N \to \infty$ \citep{kolokoltsov2010nonlinear,carmona2015MKV}. Under reasonable assumptions, we may then argue that, as $N$ tends to infinity, the empirical distribution $\hat{\mu}_t \deq N^{-1}\sum^N_{i=1}\delta_{W^i_t}$ converges to a Borel probability measure $\mu_t$ governing the distribution of the weights of a ``representative'' neuron in an infinitely wide net. The weight vector of this representative neuron will evolve according to the It\^o SDE
\begin{align}\label{eq:mean_field_SDE}
	\dif W_t = u_t \dif t + \dif B_t, \qquad W_0 \equiv 0, \, t \in [0,T]
\end{align}
where we now choose the drift $\bd{u}$ to minimize the mean-field expected cost
\begin{align}\label{eq:mean_field_cost}
	J_{\beta,T}(\bd{u}) \deq \E\left[ \frac{1}{2}\int^T_0 \|u_t\|^2 \dif t\right] + \frac{\beta}{2}R(\mu_T), \qquad \mu_T := {\rm Law}(W_T)
\end{align}
where $R(\cdot)$ is the mean-field $L^2$ risk as defined in \eqref{eq:L2risk_mf}. As will be shown below, this is equivalent to entropic regularization in path space, and this bias towards parsimony in the dynamics corresponds explicitly to an optimal control problem regarding the transfer of weights from their initialization to their distribution at the terminal time $T$. In fact, we will give two equivalent formulations of entropic regularization: a \textit{static} one, operating at the level of probability measures on the space of continuous paths from $[0,T]$ into $\Reals^d$, and a \textit{dynamic} one, pertaining to the above optimal control problem. Incidentally, this control problem has a very special structure due to the \textit{nonlinear} dependence of the terminal cost in \eqref{eq:mean_field_cost} on the target measure $\mu_T$. As such, it is an instance of a \textit{McKean--Vlasov} optimal control problem \citep{carmona2015MKV}, which requires more sophisticated machinery compared to the standard formulation of stochastic control of diffusion processes \citep{fleming1975control}. Let us denote by $\mu^\star = \mu^\star_{\beta,T}$ the probability law of $W_T$ that achieves the minimum in \eqref{eq:mean_field_cost} subject to \eqref{eq:mean_field_SDE} (leaving aside for now the questions of existence and uniqueness).

Now, it is of interest to compare the \textit{finite-time} particle dynamics \eqref{eq:particles} or its mean-field limit \eqref{eq:mean_field_SDE} to the following horizon-free dynamics of the gradient descent type:
\begin{align}\label{eq:particle_NLD}
	\dif \hat{W}^i_t = -\frac{1}{2}\left(\beta \nabla \Psi(\hat{W}^i_t; \hat{\mu}_t) + \frac{\hat{W}^i_t}{T}\right) \dif t + \dif B_t, \qquad t \ge 0,\, i \in [N]
\end{align}
for the $N$ coupled neurons or
\begin{align}\label{eq:mean_field_NLD}
	\dif \hat{W}_t = -\frac{1}{2}\left( \beta \nabla \Psi(\hat{W}_t; \mu_t) + \frac{\hat{W}_t}{T}\right)\dif t + \dif B_t, \qquad t \ge 0
\end{align}
for the mean-field limit, where
\begin{align}
	\Psi(w;\mu) := \tilde{f}(w) + \int_{\Reals^d} K(w,\tilde{w})\mu(\dif \tilde{w}),
\end{align}
$\hat{\mu}_t = N^{-1}\sum^N_{i=1}\delta_{\hat{W}^i_t}$ is the empirical distribution of $\hat{\bd{W}}_t = (\hat{W}^1_t,\ldots,\hat{W}^N_t)$, and $\mu_t$ is the law of $\hat{W}_t$. (We  note that the SDE \eqref{eq:mean_field_NLD} is also of the McKean--Vlasov type due to the explicit dependence of the drift on the marginal probability law $\mu_t$ of $\hat{W}_t$.) This is the setting studied in earlier work on two-layer neural nets in the mean-field regime \citep{mei2018meanfield,rotskoff2018meanfield,mei2019meanfield,sirignano2020meanfield_CLT}, and one of the main messages of this line of work is that, under mild regularity assumptions on the target function $f$, the activation function $\act$, and the initial distribution in \eqref{eq:mean_field_NLD}, the probability law $\mu_t$ of $\hat{W}_t$ converges weakly to a unique $\hat{\mu}^\star = \hat{\mu}^\star_{\beta,T}$ as $t \to \infty$. Regarding this, one can show the following:
\begin{enumerate}
	\item The two probability laws $\hat{\mu}^\star$ and $\mu^\star$ are equal, and the key distinction between \eqref{eq:mean_field_SDE} (with the optimal drift) and \eqref{eq:mean_field_NLD} is that the former solves a finite-time optimal control problem and hits the optimal distribution $\mu^\star$ at the prescribed finite time $T$, whereas the latter is a stochastic gradient flow that only attains the optimal distribution asymptotically as $t \to \infty$.
	\item When the product $\beta T$ is suitably small, the probability law $\mu_t$ will be $\eps$-close to $\mu^\star$ in $L^2$ Wasserstein distance provided $t \sim T \log (\frac{1}{\eps})$; above that threshold value, one can only guarantee $\eps$-closeness for $t \sim Te^{\beta}\log (\frac{1}{\eps})$.
\end{enumerate}
In particular, when it comes to showing that $\mu_t$ is close to $\mu^\star$ on a timescale comprable to $T$ when $\beta T$ is small, we take the usual route and relate \eqref{eq:mean_field_NLD} to the Langevin diffusion process, governed by
\begin{align}
	\dif V_t = - \frac{1}{2}\left(\beta \nabla \Psi(V_t; \mu^\star) + \frac{V_t}{T}\right) \dif t + \dif B_t, \qquad t \ge 0
\end{align}
such that the law of $V_t$ can also be shown to converge to $\mu^\star$ as $t \to \infty$. It was shown recently by \cite{chizat22MKV} and \cite{nitanda2022MKV} that, under certain regularity assumptions, this convergence is exponentially fast. We give an elementary proof of a weaker result, which suffices for our purposes and may be of independent interest.

\subsection{Notation} \sloppypar For any $t > 0$, we will denote by $\gamma_t$ the centered Gaussian measure on $\Reals^d$ with covariance matrix $t I_d$. The space of Borel probability measures on $\Reals^d$ will be denoted by $\fP(\Reals^d)$, and $\fP_p(\Reals^d)$, for $p \ge 1$, will stand for the set of $\mu \in \fP(\Reals^d)$ with finite $p$th moment. The $L^p$ Wasserstein distance between $\mu,\nu \in \fP_p(\Reals^d)$ is
\begin{align*}
	\Wass_p(\mu,\nu) \deq \left(\inf_{W \sim \mu, V \sim \nu} \E [\| W - V \|^p] \right)^{1/p},
\end{align*}
where the infimum is over all random elements $(W,V)$ of $\Reals^d \times \Reals^d$ with marginals $\mu$ and $\nu$, and $\| \cdot \|$ denotes the Euclidean $(\ell_2)$ norm on $\Reals^d$. For each $T > 0$, we consider the Wiener space $(\Omega_T,(\cF_t)_{t \in [0,T]},(W_t)_{t \in [0,T]},\bd{\gamma}_T)$ for $T > 0$, where $\Omega_T = C([0,T]; \Reals^d)$ is the space of continuous functions $\omega : [0,T] \to \Reals^d$; $\bd{W} = (W_t)_{t \in [0,T]}$ is the canonical coordinate process, $W_t(\omega) \deq \omega(t)$; $(\cF_t)_{t \in [0,T]}$ is the natural filtration of $\bd{W}$; and $\bd{\gamma}_T$ is the Wiener measure, under which $\bd{W}$ is the standard $d$-dimensional Brownian motion on $[0,T]$. Given any probability measure $\bd{\mu}_T$ on $\Omega_T$, we will denote by $\mu_t$ the probability law of $W_t$ under $\bd{\mu}_T$; for the Wiener measure $\bd{\gamma}_T$, this notation is consistent with our defintion of $\gamma_t$. We will use $D(\cdot \| \cdot)$ to denote the relative entropy between any two probability measures either on $\Omega$ or on $\Reals^d$. Additional notation will be introduced in the sequel as needed.

\section{Path-space free energy minimization problem}

Let $\beta, T > 0$ be given. We consider minimizing, over all $\bd{\mu}_T \in \fP(\Omega_T)$, the following free energy functional:
\begin{equation}\label{eq:fe_functional}
F_{\beta,T}(\bd{\mu}_T) := \frac{1}{2}R(\mu_T) + \frac{1}{\beta}D(\bd{\mu}_T||\bd{\gamma}_T).
\end{equation}
The idea here is to view each $\bd{\mu}_T \in \fP(\Omega)$ as the probability law of a process $\bd{W} = (W_t)_{t \in [0,T]}$ with continuous sample paths, such that $\mu_T = {\rm Law}(W_T) \in \fP(\Reals^d)$ gives a candidate probability measure for the weights of the neural net in the mean-field regime. The free energy in \eqref{eq:fe_functional} trades off the $L^2(\pi)$ risk of $\mu_T$ against the relative entropy of the path-space measure $\bd{\mu}_T$ w.r.t.\ the Wiener measure $\bd{\gamma}_T$, and the relative importance of the two terms is controlled by the parameter $\beta$. The time horizon parameter $T$ determines the variance of the Gaussian ``prior'' $\gamma_T$.

At first sight, the problem of minimizing \eqref{eq:fe_functional} appears rather intractable, involving probability measures over an infinite-dimensional function space. However, a closer examination reveals a fortuitous decoupling:  the relative entropy term can be decomposed via the chain rule, giving
\begin{equation}\label{eq:kl_chain_rule}
D(\bd{\mu}_T\|\bd{\gamma}_T) = \int_{\Reals^d} \mu_T(\dif x)\underbrace{D(\bd{\mu}^x_T\| \bd{\gamma}^x_T)}_{(*)} + D(\mu_T||\gamma_T),
\end{equation}
where $\bd{\mu}^x_T$ and $\bd{\gamma}^x_T$ are the conditional laws $\bd{\mu}_T$ and $\bd{\gamma}_T$ given $W_T = x$. Because the term $(*)$ in \eqref{eq:kl_chain_rule} corresponding to each $x \in \Reals^d$ is nonnegative and equals zero if and only if $\bd{\mu}^x_T = \bd{\gamma}^x_T$, we find that, for any given $\bd{\mu}_T$, the probability measure $\bar{\bd{\mu}}_T$ given by the mixture
\begin{equation*}
	\bar{\bd{\mu}}_T = \int_{\Reals^d} \mu_T(\dif x) \bd{\gamma}^x_T
\end{equation*}
attains a smaller value of the free energy. Indeed, since $\bar{\mu}_T = \mu_T$ and $D(\bar{\bd{\mu}}_T \|\bd{\gamma}_T) = D(\mu_T \| \gamma_T)$, we have
\begin{align*}
	F_{\beta,T}(\bd{\mu}_T) &= \frac{1}{2}R(\mu_T) + \frac{1}{\beta} D(\bd{\mu}_T \| \bd{\gamma}_T) \\
	&\ge \frac{1}{2}R(\mu_T) + \frac{1}{\beta} D(\mu_T \| \gamma_T) \\
	&= \frac{1}{2}R(\bar{\mu}_T) + \frac{1}{\beta} D(\bar{\bd{\mu}}_T \| \bd{\gamma}_T) \\
	&= F_{\beta,T}(\bar{\bd{\mu}}_T).
\end{align*}
We can summarize this observation in the following:
\begin{proposition} The minimum value of the free energy satisfies
	\begin{align}
		\inf_{\bd{\mu}_T \in \fP(\Omega_T)} F_{\beta,T}(\bd{\mu}_T) &= \inf_{\bd{\mu}_T \in \fP(\Omega_T)} F_{\beta,T}(\bar{\bd{\mu}}_T) \label{eq:fe_reduction_1} \\
		&= \inf_{\mu \in \fP(\Reals^d)} \left[ \frac{1}{2}R(\mu) + \frac{1}{\beta} D(\mu \| \gamma_T)\right] \label{eq:fe_reduction_2}
	\end{align}
Moreover, if $\mu^\star \in \fP(\Reals^d)$ achieves the infimum in \eqref{eq:fe_reduction_2}, then
\begin{align}\label{eq:opt_pathspace_mu}
\bd{\mu}^\star_T = \int_{\Reals^d} \mu^\star(\dif x) \bd{\gamma}^x_T
\end{align}
achieves the infimum in the left-hand side of \eqref{eq:fe_reduction_1}.
\end{proposition}
Thus, the problem of minimizing the free energy \eqref{eq:fe_functional} over path-space measures $\bd{\mu}_T \in \fP(\Omega_T)$ reduces to minimizing the functional $\frac{1}{2}R(\mu) + \frac{1}{\beta}D(\mu \| \gamma_T)$ over probability measures $\mu \in \fP(\Reals^d)$. Abusing notation slightly, we will denote this functional by $F_{\beta,T}(\mu)$ as well. Assuming that a minimizing $\mu^\star$ exists and is unique, the optimal path-space measure $\bd{\mu}^\star_T$ in \eqref{eq:opt_pathspace_mu} has a very specific structure. Indeed, we recognize $\bd{\gamma}^x_T$ as the probability law of the standard \textit{Brownian bridge} pinned to the point $x$ at time $T$, and therefore $\bd{\mu}^\star_T$ is the probability law of the standard Brownian motion conditioned to have the marginal distribution $\mu^\star$ at time $T$ \citep{follmer1985reversal}. Moreover, we can readily compute the Radon--Nikodym derivative of $\bd{\mu}^\star_T$ w.r.t.\ the Wiener measure $\bd{\gamma}_T$. Since $D(\mu^\star \|\gamma_T) \le \beta F_{\beta,T}(\mu^\star) < \infty$, $\mu^\star$ is absolutely continuous w.r.t.\ $\gamma_T$. Therefore, for any bounded measurable function $h : \Omega_T \to \Reals$,
\begin{align*}
	\E_{\bd{\mu}^\star_T}[h(\bd{W})] &= \int_{\Reals^d} \mu^\star(\dif x) \int_\Omega \bd{\gamma}^x_T(\dif \omega) h(\omega) \\
	&= \int_{\Reals^d} f(x) \gamma_T(\dif x) \int_\Omega \bd{\gamma}^x_T(\dif \omega) h(\omega) \\
	&= \int_\Omega \bd{\gamma}_T(\dif \omega) \frac{\dif\mu^\star}{\dif\gamma_T}(\omega(T)) h(\omega) \\
	&= \E_{\bd{\gamma}_T}\left[\frac{\dif\mu^\star}{\dif\gamma_T}(W_T) h(\bd{W})\right],
\end{align*}
which gives
\begin{align*}
	\frac{\dif\bd{\mu}^\star_T}{\dif\bd{\gamma}_T}(\bd{W}) = \frac{\dif\mu^\star}{\dif\gamma_T}(W_T).
\end{align*}

\subsection{The static formulation:  structural result}

We now show that the infimum in \eqref{eq:fe_reduction_2} is achieved by a unique probability measure $\mu^\star \in \fP(\Reals^d)$. This result, stated as Theorem~\ref{thm:static} below, relies on the following assumptions (also made by \citet{mei2018meanfield}):

\begin{assumption}\label{as:bounded} The target function $f : \Reals^p \to \Reals$ and the activation function $\act : \Reals^p \times \Reals^d \to \Reals$ are bounded: $\| f \|_\infty, \| \act \|_\infty \le \kappa_1$. Moreover, for each $w \in \Reals^d$ the gradient $\nabla_w \act(X;w)$ is $\kappa_1$-subgaussian when $X \sim \pi$.
\end{assumption}

\begin{assumption}\label{as:Lipschitz} The functions $\tilde{f}$ and $K$ are differentiable and Lipschitz-continuous, with Lipschitz-continuous gradients: $\| \nabla \tilde{f}(w)\| \le \kappa_2$, $\| \nabla K(w,\tilde{w})\| \le \kappa_2$, $\| \nabla \tilde{f}(w) - \nabla \tilde{f}(w') \| \le \kappa_2 \| w-w' \|$, $\| \nabla K(w,\tilde{w}) - \nabla K(w',\tilde{w}') \| \le \kappa_2 \|(w,\tilde{w})-(w',\tilde{w}') \|$.
\end{assumption}
\noindent Throughout the paper, we will use $\kappa$ to denote a generic quantity that grows like $O({\rm poly}(\max\{\kappa_1,\kappa_2\}))$ and $c$ to denote a generic absolute constant. The values of $\kappa$ and $c$ may change from line to line.

\begin{theorem}\label{thm:static} Under Assumptions~\ref{as:bounded} and \ref{as:Lipschitz}, the free energy $F_{\beta,T}(\cdot)$ admits a unique minimizer $\mu^\star = \mu^\star_{\beta,T}$, such that the following hold:
	\begin{enumerate}
		\item $\mu^\star$ is absolutely continuous w.r.t.\ $\gamma_T$ and satisfies the {\em Boltzmann fixed-point condition}
		\begin{align}\label{eq:Boltzmann_FP}
			\mu^\star(\dif w) = \frac{1}{Z^\star}\exp\left(-\beta\Psi(w;\mu^\star)\right)\gamma_T(\dif w),
		\end{align}
		where the potential $\Psi : \Reals^d \times \fP_2(\Reals^d) \to \Reals$ is given by
		\begin{align*}
			\Psi(w; \mu) \deq \tilde{f}(w) + \int_{\Reals^d} K(w,\tilde{w})\mu(\dif \tilde{w})
		\end{align*}
		and $Z^\star = Z(\beta,T; \mu^\star)$ is the normalization constant.
		\item The risk of $\mu^\star$ is bounded by
		\begin{align}\label{eq:FE_upper_bound}
			R(\mu^\star) \le 2\FE_{\beta,T}(\mu^\star) \le \inf_{\mu \in \fP_2(\Reals^d)} \left[R(\mu) + \frac{1}{\beta T}M^2_2(\mu)\right] + \frac{d}{\beta T} \log(2\kappa\beta T+1),
		\end{align}
		where $M^2_2(\mu) \deq \int_{\Reals^d}\|w\|^2 \mu(\dif w)$ is the second moment of $\mu$.
		\item The relative entropy and the squared $L^2$ Wasserstein distance between $\mu^\star$ and the Gaussian prior $\gamma_T$ are bounded by
		\begin{align}\label{eq:entropy_W2}
			D(\mu^\star \| \gamma_T) \le \kappa T\beta^2 \qquad \text{and} \qquad \Wass^2_2(\mu^\star,\gamma_T) \le \kappa T^2\beta^2.
		\end{align}
		\item If $f = \hat{f}(\cdot;\mu^\circ)$ for some $\mu^\circ \in \fP_2(\Reals^d)$, then
		\begin{align}\label{eq:realizable}
			R(\mu^\star) \le \frac{2}{\beta}D(\mu^\circ\|\gamma_T).
		\end{align}
	\end{enumerate}

\end{theorem}

\begin{proof}
The risk $R(\gamma_T)$ is finite by virtue of Assumption~\ref{as:bounded}, and evidently
 	\begin{align*}
		0 \le \inf_{\mu \in \fP_2(\Reals^d)} \FE_{\beta,T}(\mu) \le F_{\beta,T}(\gamma_T) = \frac{1}{2}R(\gamma_T).
	\end{align*}
Therefore, we can restrict the minimization to the set
\begin{align*}\fM \deq \left\{ \mu \in \fP_2(\Reals^d) : D(\mu \| \gamma_T) \le \frac{\beta}{2}R(\gamma_T)\right\},
\end{align*}
which is weakly compact by Lemma~\ref{lm:weak_compactness}. By Assumptions~\ref{as:bounded} and \ref{as:Lipschitz}, $F_{\beta,T}$ is a weakly lower-semicontinuous functional, and therefore attains an infimum on $\fM$. Uniqueness follows from the fact that $F_{\beta,T}$ is a positive linear combination of a convex functional $\mu \mapsto R(\mu)$ and a strictly convex functional $\mu \mapsto D(\mu \| \gamma_T)$, and is therefore strictly convex. Hence, $F_{\beta,T}$ has a unique minimizer $\mu^\star = \mu^\star_{\beta,T} \in \fP_2(\Reals^d)$.

To prove that $\mu^\star$ satisfies the Boltzmann fixed-point condition \eqref{eq:Boltzmann_FP}, we proceed analogously to the proof of Lemma~10.3 of \citet{mei2018meanfield}. Let $\lambda$ denote the Lebesgue measure on $\Reals^d$. We first show that $\mu^\star$ has an almost everywhere positive density w.r.t.\ $\lambda$. 

Since $D(\mu^\star \| \gamma_T) < \infty$, $\mu^\star$ is absolutely continuous w.r.t.\ $\lambda$. Thus, the density $\rho^\star \deq \frac{\dif \mu^\star}{\dif \lambda}$ exists. To show that $\rho^\star > 0$ almost everywhere, suppose, by way of contradiction, that there exists a  set $K \subset \Reals^d$ with $\lambda(K) > 0$, such that $\rho^\star = 0$ on $K$. Without loss of generality, we may assume that $K$ is compact (otherwise, we can replace $K$ by its intersection with a ball of suitably large radius). Let $\rho$ be an arbitrary probability density supported on $K$, such that the differential entropy
\begin{align*}
	h(\rho) = - \int_{\Reals^d} \rho(w)\log \rho(w) \dif w
\end{align*}
is finite (for instance, we can take $\rho(w) = \frac{\1_K(w)}{\lambda(K)}$). Consider the mixture $\mu^\eps = (1-\eps)\mu^\star + \eps \mu$ for some $\eps \in (0,1)$, where $\dif \mu  \deq \rho \dif\lambda$. Then
\begin{align*}
	 R(\mu^\eps) - R(\mu^\star) &= 2\eps \int_{\Reals^d}\tilde{f}(w)[\rho(w)-\rho^\star(w)]\dif w \nonumber\\
	&\qquad + 2\eps(1-\eps) \int_{\Reals^d \times \Reals^d} K(w,\tilde{w}) [\rho^\star(w)\rho(\tilde{w}) - \rho^\star(w)\rho^\star(\tilde{w})] \dif w \dif \tilde{w} \nonumber\\
	& \qquad + \eps^2 \int_{\Reals^d \times \Reals^d} K(w,\tilde{w}) [\rho(w)\rho(\tilde{w}) - \rho^\star(w)\rho^\star(\tilde{w})] \dif w \dif \tilde{w} 
\end{align*}
and
\begin{align*}
	& D(\mu^\eps \| \gamma_T) - D(\mu^\star \| \gamma_T) \\
	& = h(\rho^\star) - h((1-\eps)\rho^\star + \eps \rho) + \frac{\eps}{2T} \int_{\Reals^d} \|w\|^2 [\rho(w)-\rho^\star(w)] \dif w \\
	&= \eps[h(\rho^\star) - h(\rho)] + \eps \log \eps + (1-\eps) \log (1-\eps) +  \frac{\eps}{2T} \int_{\Reals^d} \|w\|^2 [\rho(w)-\rho^\star(w)] \dif w,
\end{align*}
where we have used the fact that $\mu^\star$ and $\mu$ are mutually singular, so that
\begin{align*}
	h((1-\eps)\rho^\star + \eps \rho) = (1-\eps)h(\rho^\star) + \eps h(\rho) - \eps \log \eps - (1-\eps)\log (1-\eps).
\end{align*}
Thus, for $\eps \in (0,1)$ we can upper-bound the difference $\FE_{\beta,T}(\mu^\eps) - \FE_{\beta,T}(\mu^\star)$ by an expression of the form
\begin{align*}
	C_1 \eps + C_2 (\eps \log \eps + (1-\eps) \log (1-\eps))
\end{align*}
for some positive constants $C_1,C_2 > 0$, which takes negative values for all sufficiently small $\eps > 0$. As a consequence, we see that, for all sufficiently small $\eps > 0$,
\begin{align*}
	\FE_{\beta,T}(\mu^\eps) - \FE_{\beta,T}(\mu^\star) < 0,
\end{align*}
which contradicts the optimality of $\mu^\star$. Thus, $\rho^\star > 0$ almost everywhere. 

Next, we show that $\Psi(w; \mu^\star) + \frac{1}{\beta} \big(\log \rho^\star(w) + \frac{\|w\|^2}{2T}\big)$ is constant almost everywhere. Fix some $\eps_0 > 0$ and consider the set
\begin{align*}
	S_{\eps_0} \deq \left\{ w \in \Reals^d : \rho^\star(w) \ge \eps_0 \text{ and } \|w\| \le \eps_0 \right\}.
\end{align*}
Fix a differentiable function $v$, such that (i) $v = 0$ on $\Reals^d \setminus S_{\eps_0}$, (ii) $\|v\|_\infty \le 1$, (iii) $\int_{\Reals^d} v(w)\dif w = 0$. Then $\rho^\star + \eps v$ is a probability density for all $\eps \in [-\eps_0,\eps_0]$, so for the probability measures $\dif\mu^\eps \deq \dif\mu^\star + \eps v \dif \lambda$ we have
\begin{align*}
	\lim_{\eps \to 0} \frac{\FE_{\beta,T}(\mu^\eps) - \FE_{\beta,T}(\mu^\star)}{\eps} =\int_{\Reals^d} \left[\Psi(w; \mu^\star) + \frac{1}{\beta}\left(\log \rho^\star(w) + \frac{\|w\|^2}{2T}\right)\right] v(w) \dif w \ge 0.
\end{align*}
Repeating the same argument with $-v$ instead of $v$, we see that
\begin{align*}
	\int_{\Reals^d} \left[\Psi(w; \mu^\star) + \frac{1}{\beta}\left(\log \rho^\star(w) + \frac{\|w\|^2}{2T}\right)\right] v(w) \dif w = 0
\end{align*}
for all $v$ satisfying conditions (i)--(iii) above. This implies that $\Psi(w; \mu^\star) + \frac{1}{\beta}\left(\log \rho^\star(w) + \frac{\|w\|^2}{2T}\right) = {\rm const}$ for all $w \in S_{\eps_0}$. Since $\lambda (\Reals^d \setminus \cup_{\eps_0 > 0}S_{\eps_0}) = 0$, we see that
\begin{align}\label{eq:Boltzmann_FP_2}
	\Psi(w; \mu^\star) + \frac{1}{\beta}\left(\log \rho^\star(w) + \frac{\|w\|^2}{2T}\right) = \xi(\beta,T; \mu^\star)
\end{align}
holds almost everywhere for some constant $\xi(\beta,T; \mu^\star)$. Since $\mu^\star$ is a probability measure, we see that $\xi(\beta,T; \mu^\star) = - \frac{1}{\beta} Z(\beta,T; \mu^\star)$. Exponentiating both sides of \eqref{eq:Boltzmann_FP_2} and rearranging gives the Boltzmann fixed-point equation \eqref{eq:Boltzmann_FP}.

To prove \eqref{eq:FE_upper_bound}, we proceed analogously to the proof of Lemma~10.5 of \citet{mei2018meanfield}. Pick some $\eps > 0$, to be chosen later. Let $G,\tilde{G}$ be two independent samples from $\gamma_1$. Then, for any $\mu \in \fP_2(\Reals^d)$,
\begin{align*}
R(\mu * \gamma_\eps) - R(\mu) & = 2\int_{\Reals^d} \E[\tilde{f}(w + \sqrt{\eps}G)-\tilde{f}(w)] \mu(\dif w) \nonumber\\
& \qquad + \int_{\Reals^d \times \Reals^d} \E[K(w+\sqrt{\eps}G,\tilde{w} + \sqrt{\eps}\tilde{G}) - K(w,\tilde{w})]\mu(\dif w)\mu(\dif\tilde{w}),
\end{align*}
where, using the intermediate value theorem, we can write
\begin{align}
	\E[\tilde{f}(w+\sqrt{\eps}G)] &= \tilde{f}(w) + \sqrt{\eps}\, \E[\langle \nabla \tilde{f}(w), G \rangle] + \frac{\eps}{2} \E[\langle G, \nabla^2 \tilde{f}(\xi) G\rangle]
\end{align}
for some (random) point $\xi$ in $\Reals^d$. Since $\nabla^2 \tilde{f}$ is bounded by Assumption~\ref{as:Lipschitz}, we conclude that
\begin{align*}
|\E[\tilde{f}(w+\sqrt{\eps}G)]-\tilde{f}(w)| \le \kappa \eps d.
\end{align*}
An analogous argument gives
\begin{align*}
	|\E[K(w+\sqrt{\eps}G,\tilde{w}+\sqrt{\eps}\tilde{G})]-K(w,\tilde{w})| \le \kappa \eps d.
\end{align*}
Thus,
\begin{align}\label{eq:R_bound}
	R(\mu * \gamma_\eps) \le R(\mu) + \kappa d\eps.
\end{align}
Moreover, for any $\mu \in \fP_2(\Reals^d)$, the convolution $\mu * \gamma_\eps$ has a smooth density w.r.t.\ $\lambda$, say, $\rho_\eps$, and therefore the differential entropy $h(\rho_\eps)$ is well-defined. Consequently,
\begin{align*}
	D(\mu * \gamma_\eps \| \gamma_T) &= - h(\rho_\eps) + \frac{1}{2T}\int_{\Reals^d} \|w\|^2 \rho_\eps(w)\dif w + \frac{d}{2}\log(2\pi T) \\
	&= - h(\rho_\eps) + \frac{1}{2T}M^2_2(\mu) + \frac{\eps d}{2T} + \frac{d}{2}\log(2\pi T).
\end{align*}
Since differential entropy increases under convolution \citep{cover2006infotheory}, we can further estimate
\begin{align*}
	h(\rho_\eps) \ge h(\gamma_\eps) = \frac{d}{2}\log(2\pi e\eps),
\end{align*}
which gives
\begin{align}\label{eq:ent_bound}
	D(\mu * \gamma_\eps \| \gamma_T)  \le \frac{1}{2T}M^2_2(\mu) + \frac{\eps d}{2T} + \frac{d}{2}\log \frac{T}{e\eps}.
\end{align}
From Eqs.~\eqref{eq:R_bound} and \eqref{eq:ent_bound}, it follows that, for any $\mu \in \fP_2(\Reals^d)$,
\begin{align*}
	\FE_{\beta,T}(\mu * \gamma_\eps) &= \frac{1}{2}R(\mu * \gamma_\eps) + \frac{1}{\beta} D(\mu * \gamma_\eps \| \gamma_T) \\
	&\le \frac{1}{2}\left[ R(\mu) + \frac{1}{\beta T}M^2_2(\mu)\right] + \kappa d\eps + \frac{d\eps}{2\beta T} + \frac{d}{2\beta}\log \frac{T}{e\eps}.
\end{align*}
By virtue of the optimality of $\mu^\star$, we obtain 
\begin{align*}
	\FE_{\beta,T}(\mu^\star) \le \frac{1}{2}\inf_{\mu \in \fP_2(\Reals^d)} \left[ R(\mu) + \frac{1}{\beta T}M^2_2(\mu)\right] + \frac{d}{2\beta T}\inf_{\eps \ge 0} \left[ (2\kappa\beta T+1)\eps +  T\log \frac{T}{e\eps} \right].
\end{align*}
Optimizing over $\eps$, we get \eqref{eq:FE_upper_bound}.

Next, we move on to \eqref{eq:entropy_W2}. From the Boltzmann fixed-point equation \eqref{eq:Boltzmann_FP}, it follows that that $\nabla \log \frac{\dif \mu^\star}{\dif \gamma_T} = - \beta \nabla \Psi(\cdot; \mu^\star)$. Therefore, by the log-Sobolev inequality \eqref{eq:LSI} for $\gamma_T$,
\begin{align*}
	D(\mu^\star \| \gamma_T) &\le \frac{T}{2} I(\mu^\star \|\gamma_T)\\
	&= \frac{T\beta^2}{2} \int_{\Reals^d}  \| \nabla \Psi(\cdot; \mu^\star)\|^2 \dif\mu^\star\\
	&\le \kappa T \beta^2.
\end{align*}
Moreover, by Talagrand's entropy-transport inequality \eqref{eq:Talagrand},
\begin{align*}
	\Wass^2_2(\mu^\star,\gamma_T) \le 2T D(\mu^\star \| \gamma_T) \le \kappa T^2\beta^2.
\end{align*}
Finally, if $f = \hat{f}(\cdot;\mu^\circ)$ for some $\mu^\circ \in \fP_2(\Reals^d)$, then $R(\mu^\circ) = 0$, and evidently
\begin{align*}
	R(\mu^\star) \le 2\FE_{\beta,T}(\mu^\star) \le 2\FE_{\beta,T}(\mu^\circ) = \frac{2}{\beta}D(\mu^\circ \|\gamma_T),
\end{align*}
which concludes the proof.
\end{proof}

From the risk bound \eqref{eq:FE_upper_bound}, it readily follows that, for any fixed $T > 0$,
\begin{align*}
	\lim_{\beta \to \infty} \inf_{\mu \in \fP_2(\Reals^d)} \FE_{\beta,T}(\mu) = \frac{1}{2}\inf_{\mu \in \fP_2(\Reals^d)} R(\mu).
\end{align*}
However, a more intriguing message of Theorem~\ref{thm:static} is that it is also meaningful to consider the regime where either $\beta$ or $T$ (or both) are small. For example, we could take $T = \eps$ and $\beta = \eps^{-1/2}$ for some $\eps \in (0,1)$. As an illustration, consider the realizable case, i.e., when the target function $f$ is equal to  $\hat{f}(\cdot; \mu^\circ)$ for some $\mu^\circ \in \fP_2(\Reals^d)$. Then Eq.~\eqref{eq:realizable} gives
\begin{align*}
	R(\mu^\star) \le 2\eps^{1/2} D(\mu^\circ \| \gamma_T),
\end{align*}
so $R(\mu^\star)$ will be on the order of $\eps^{1/2}$ for $D(\mu^\circ \| \gamma_T) = O(1)$.  This regime can be viewed as a mean-field counterpart of the notion that, for certain target functions $f$, with high probability there exist good neural-net approximations near a random Gaussian initialization \citep{allenzhu2019over}. Moreover, such a good approximation can be obtained from the random Gaussian initialization by applying a transport mapping to the weights \citep{ji2020transport}. The following theorem gives a precise statement:

\begin{theorem}\label{thm:transport} Let $\beta, T > 0$ be given, where $0 < \beta T < \frac{1}{c\kappa_2}$. Then there exists a Lipschitz-continuous transportation mapping $\Phi : \Reals^d \to \Reals^d$ such that all of the following holds with probability at least $1-\delta$ for a tuple $\bd{W} = (W^1,\ldots,W^N)$ of i.i.d.\ draws from $\gamma_T$:
	\begin{enumerate}
		\item The neural net with transported weights $\hat{f}_N(\cdot; \Phi(\bd{W})) \deq \frac{1}{N}\sum^N_{i=1}\act(\cdot; \Phi(W^i))$ satisfies
		\begin{align}\label{eq:transport_risk}
			\|f - \hat{f}_N(\cdot; \Phi(\bd{W}))\|_{L^2(\pi)} \le \| f - \hat{f}(\cdot; \mu^\star_{\beta,T})\|_{L^2(\pi)} + \kappa \sqrt{\frac{\log(1/\delta)}{N}};
		\end{align}
		\item The transported weights $\Phi(W^i)$ are uniformly close to the i.i.d.\ Gaussian weights $W^i$:
		\begin{align}\label{eq:transport_distance}
			\max_{i \in [N]} \| \Phi(W^i) - W^i \| \le \kappa \beta T + \kappa \sqrt{T(\log N + \log (1/\delta))};
		\end{align}
		\item The transported neural net $\hat{f}_N(\cdot; \Phi(\bd{W}))$ and the random Gaussian neural net $\hat{f}_N(\cdot; \bd{W})$ are close in $L^2(\pi)$ norm:
	\begin{align}\label{eq:transport_vs_init}
		\left\| \hat{f}_N(\cdot; \Phi(\bd{W})) - \hat{f}_N(\cdot; \bd{W}) \right\|^2_{L^2(\pi)} \le \kappa\beta T + \kappa  \sqrt{\frac{T\log(1/\delta)}{N}}.
	\end{align}
	\end{enumerate}
\end{theorem}

\begin{proof} Let $\mu^\star = \mu^\star_{\beta,T}$ be the global minimizer of the free energy $\FE_{\beta,T}$. By Theorem~\ref{thm:static}, $\mu^\star$ has an almost everywhere positive density w.r.t.\ $\gamma_T$, and therefore has an almost everywhere positive density $\rho^\star$ w.r.t.\ the Lebesgue measure on $\Reals^d$:
	\begin{align*}
		\rho^\star(w) = \frac{1}{Z} \exp\left(-V(w;\mu^\star)\right), \qquad V(w;\mu^\star) \deq \frac{1}{2T}\|w\|^2 + \beta \Psi(w;\mu^\star).
	\end{align*}
By the theory of optimal transport \citep{Villani_topics}, there exists a mapping $\Phi : \Reals^d \to \Reals^d$, such that (i) $\mu^\star$ is equal to the pushforward of $\gamma_T$ by $\Phi$ and (ii) $\Wass^2_2(\mu^\star,\gamma_T) = \E[\|\Phi(W)-W\|^2]$ for $W \sim \gamma_T$. By Assumption~\ref{as:Lipschitz}, the Hessian of $V(w; \mu^\star)$ satisfies
\begin{align*}
	\nabla^2 V(w;\mu^\star) = \frac{1}{T}I_d + \beta \nabla^2 \Psi(w; \mu^\star) \succeq \frac{1}{T}(1-c\kappa_2\beta T)I_d.
\end{align*}
Since $1-c\kappa_2\beta T > 0$, $V(w;\mu^\star)$ is strongly convex, and therefore the optimal transport mapping $\Phi$ is Lipschitz-continuous by Caffarelli's regularity theorem \citep{caffarelli2000regularity,kolesnikov2010transport}:
\begin{align}\label{eq:T_Lip}
	\|\Phi(w)-\Phi(\tilde{w})\| \le \frac{1}{1-c\kappa_2\beta T} \|w-\tilde{w}\|.
\end{align}
Now, let $W^1,\ldots,W^N$ be i.i.d.\ samples from $\gamma_T$. By the triangle inequality,
\begin{align*}
	\| f - \hat{f}_N(\cdot; \Phi(\bd{W})) \|_{L^2(\pi)} &\le \| f - \hat{f}(\cdot; \mu^\star) \|_{L^2(\pi)} + \| \hat{f}(\cdot;\mu^\star) - \hat{f}_N(\cdot; \Phi(\bd{W})) \|_{L^2(\pi)}.
\end{align*}
Since $\Phi$ is the (optimal) map that transports $\gamma_T$ to $\mu^\star$, $\Phi(W^1),\ldots,\Phi(W^N)$ are i.i.d.\ according to $\mu^\star$, so in particular $\E[\act(\cdot;\Phi(W^i))] = \hat{f}(\cdot; \mu^\star)$. Therefore, we can apply the high-probability version of Maurey's lemma due to \citet{ji2020transport} (reproduced in Appendix~\ref{app:lemmas}) to the functions $g(\cdot;w) \deq \act(\cdot;T(w))$ to conclude that
\begin{align*}
	 \| \hat{f}(\cdot;\mu^\star) - \hat{f}_N(\cdot; \Phi(\bd{W})) \|_{L^2(\pi)} &\le c\sup_{w} \|\act(\cdot; \Phi(w))\|_{L^2(\pi)} \sqrt{\frac{\log(1/\delta)}{N}} \\
	 &\le \kappa \sqrt{\frac{\log(1/\delta)}{N}}
\end{align*}
with probability at least $1-\delta$. This proves \eqref{eq:transport_risk}.

Consider now the mapping $\Delta(w) \deq \|\Phi(w)-w\|$. By Jensen's inequality, Theorem~\ref{thm:static}, and the optimality of $\Phi$,
\begin{align*}
	\E[\Delta(W^i)] = \E[\|\Phi(W^i)-W^i\|] \le \sqrt{\E[\|\Phi(W^i)-W^i\|^2]} = \Wass_2(\mu^\star,\gamma_T) \le \kappa\beta T.
\end{align*}
Moreover, from \eqref{eq:T_Lip} we see that $\Delta$ is Lipschitz-continuous:
\begin{align*}
	|\Delta(w)-\Delta(\tilde{w})| &= \left|\|\Phi(w)-w\| - \|\Phi(\tilde{w})-\tilde{w}\|\right| \\
	&\le \|\Phi(w)-\Phi(w')\| + \|w-w'\| \\
	&\le L \|w-w'\|, 
\end{align*}
where $L \deq \frac{2-c\kappa_2\beta T}{1-c\kappa_2\beta T} \le 3$ provided $\beta T$ is smaller than $ \frac{1}{2c\kappa_2}$. Therefore, the following Gaussian concentration inequality holds for every $i \in [N]$:
\begin{align*}
	\P \left\{ \Delta(W^i) \ge \E[\Delta(W^i)] + r\right\} \le e^{-\frac{r^2}{18T}}, \qquad \forall r > 0
\end{align*}
\citep[Theorem~5.6]{Boucheron_etal_book}. By the union bound, \eqref{eq:transport_distance} holds with probability at least $1-\delta$.

Finally, let us consider the function
\begin{align*}
	\Gamma(\bd{w}) &\deq \left\| \hat{f}_N(\cdot;\Phi(\bd{w})) - \hat{f}_N(\cdot; \bd{w}) \right\|^2_{L^2(\pi)}  \\
	&= \frac{1}{N^2} \left\|\sum^N_{i=1}\act(\cdot; \Phi(w^i)) - \sum^N_{i=1}\act(\cdot; w^i) \right\|^2_{L^2(\pi)}.
\end{align*}
Using the definition and the properties of $K(w,\tilde{w})$, we have
\begin{align*}
	\Gamma(\bd{w}) &= \frac{1}{N^2} \sum^N_{i=1}\sum^N_{j=1} [K(w^i,w^j) - K(w^i,\Phi(w^j))] \nonumber\\
	& \qquad + \frac{1}{N^2}\sum^N_{i=1}\sum^N_{j=1} [K(\Phi(w^i),\Phi(w^j))-K(w^i,\Phi(w^j))] \\
	& \le \frac{\kappa}{N} \sum^N_{i=1} \| \Phi(w^i) - w^i \| \\
	& = \frac{\kappa}{N} \sum^N_{i=1} \Delta(w^i).
\end{align*}
Consequently, for any $r > 0$,
\begin{align*}
	\P\left\{ \Gamma(\bd{W}) \ge \kappa (\E[\Delta(W^1)] + r)\right\} &\le \P \left\{ \frac{1}{N}\sum^N_{i=1}\Delta(W^i) \ge \E[\Delta(W^1)] + r \right\} \\
	&\le e^{-\frac{Nr^2}{18T}}.
\end{align*}
Since $\E[\Delta(W^1)] \le \kappa\beta T$, we see that \eqref{eq:transport_vs_init} holds with probability at least $1-\delta$.
\end{proof}

\subsection{The dynamic formulation:  a McKean--Vlasov optimal control problem}

Our next order of business is to present a stochastic control interpretation of the problem of minimizing the free energy \eqref{eq:fe_functional}. Given an admissible \textit{control}, i.e., any progressively measurable $\Reals^d$-valued process $\bd{u} = (u_t)_{t \in [0,T]}$ satisfying
\begin{align*}
	\E\left[\int^T_0 \|u_t\|^2\dif t\right] < \infty,
\end{align*}
we consider the It\^o SDE
\begin{align}\label{eq:MKV_dynamics}
	\dif W_t = u_t \dif t + \dif B_t, \qquad W_0 \equiv 0;\, t \in [0,T]
\end{align}
We wish to choose $\bd{u}$ to minimize the following mean-field counterpart to the cost \eqref{eq:particle_cost}:
\begin{align}\label{eq:MKV_cost}
	J_{\beta,T}(\bd{u}) \deq \E\left[\frac{1}{2}\int^T_0 \|u_t\|^2 \dif t\right] + \frac{\beta}{2}R(\mu^{\bd{u}}_T),
\end{align}
where $\mu^{\bd{u}}_t$ is the probability law of $W_t$ for $t \in [0,T]$. To motivate this dynamic optimization problem, consider first the case of zero drift: $u_t \equiv 0$ for all $t \in [0,T]$. Then the resulting process $\bd{W}$ is simply the Brownian motion $\bd{B}$, so that in particular $W_T = B_T \sim \gamma_T$, and the resulting expected cost is proportional to the $L^2(\pi)$ risk of the Gaussian prior $\gamma_T$. By adding a nonzero drift $\bd{u}$, we perturb the Brownian path $\bd{B}$, and the expected cost \eqref{eq:MKV_cost} captures the trade-off between the strength of this perturbation and the $L^2(\pi)$ risk of $\hat{f}(\cdot; \mu^{\bd{u}}_T)$.

Specifically, we will construct a flow of measures $\bd{\mu}^\star = \{\mu^\star_t\}_{t \in [0,T]}$ with densities $\rho^\star_t$, such that: (i) $\mu^\star_0 = \delta_0$ (the Dirac measure concentrated at the origin), (ii) $\mu^\star_T = \mu^\star$ (the unique minimizer of the free energy $\FE_{\beta,T}$), and (iii) the evolution of $\bd{\mu}^\star$ is governed by a system of two coupled nonlinear PDEs,
\begin{subequations}
\begin{align}
	\partial_t \rho^\star_t(w) &= \nabla_w \cdot \left(\rho^\star_t(w) \nabla_w V^\star(w,t)\right) + \frac{1}{2} \Delta_w \rho^\star_t(w) \label{eq:FPE_optimal}\\
	\partial_t V^\star(w,t) &= - \frac{1}{2}\Delta_w V^\star(w,t) + \frac{1}{2} \| \nabla_w V^\star(w,t)\|^2 \label{eq:HJB_optimal}
\end{align}
\end{subequations}
for $(w,t) \in \Reals^d\times [0,T]$, where \eqref{eq:FPE_optimal} has the Dirac delta initial condition $\rho^\star_0(w) = \delta(w)$ and \eqref{eq:HJB_optimal} has terminal condition $V^\star(w,T) = \beta\big(R_0 + \int_{\Reals^d}\tilde{f}\dif\mu^\star_T + \Psi(w; \mu^\star_T)\big)$. The forward PDE \eqref{eq:FPE_optimal} is the Fokker--Planck equation, and the backward PDE \eqref{eq:HJB_optimal} is the Hamilton--Jacobi--Bellman equation.

The problem of minimizing $D(\bd{\mu}_T\|\bd{\gamma}_T)$, subject to the attainment of a specific density at time $T$, is the classical \emph{Schr\"odinger bridge problem} (see \citet{chen2016bridges} and references therein). In particular, for a diffusion process
$$
\dif W_t = u_t\dif t + \dif B_t,\quad X_0 = 0,\ t\in [0,T]
$$
the objective is to minimize $\E\big[\frac{1}{2}\int_0^T \|u_t\|^2\dif t\big]$ subject to $W_T\sim \mu$ for some desired $\mu$. By Girsanov's theorem, the control cost and the path-space divergence are precisely the same quantity:
$$
\E\Bigg[\frac{1}{2}\int_0^T \|u_t\|^2\dif t\Bigg] = D\big(\Law(W_{[0,T]})\|\Law(B_{[0,T]})\big).$$ 
However, in a problem where the terminal cost depends on the terminal measure, this dependency propagates through the entire interval over which the process runs, leading to an instance of \emph{McKean-Vlasov dynamics}, for which it is generally impossible to obtain the optimal control in closed form. Somewhat surprisingly, though, in this case we can obtain an exact characterization for both the optimal control law and the optimal cost in terms of the minimizer $\mu^\star$ of the free energy $F_{\beta,T}$, and the Boltzmann fixed-point condition \eqref{eq:Boltzmann_FP} plays the key role in guaranteeing that the corresponding forward-backward system admits a solution.

\subsection{Solution of the optimal control problem}

\begin{theorem}\label{thm:dynamic} Let $\mu^\star$ be the (unique) minimizer of the free energy $\FE_{\beta,T}(\mu)$. Then the optimal controlled process solves the It\^o SDE
	\begin{align}\label{eq:optimal_SDE}
		\dif W_t = - \nabla_w V^\star(W_t,t)\dif t + \dif B_t, \qquad t \in [0,T];\, \, W_0 = 0
		\end{align}
		where
	\begin{align}\label{eq:value_function}
		V^\star(w,t) \deq -\log \E\big[\exp\left(-\beta \Psi(B_T;\mu^\star)\right) \big|  B_{t} = w\big].
	\end{align}
	Moreover, under \eqref{eq:optimal_SDE}, $W_T$ is distributed according to $\mu^\star$, and the above optimal control achieves
	\begin{align*}
		\inf_{\bd{u}} J_{\beta,T}(\bd{u}) = \beta \FE_{\beta,T}(\mu^\star).
	\end{align*}
\end{theorem}

\begin{proof}

We first note that the control cost $J_{\beta,T}(\bd{u})$ can be expressed as
	\begin{align*}
		J_{\beta,T}(\bd{u}) = \E\left[\frac{1}{2}\int^T_0 \|u_t\|^2\dif t + \frac{\beta}{2}\bar{c}(W^{\bd{u}}_T,\mu^{\bd{u}}_T)\right],
	\end{align*}
	where
	\begin{align*}
		\bar{c}(w;\mu) &\deq \E_\pi[(f(X)-\act(X;w))(f(X)-\hat{f}(X;\mu))] \\
		&= R_0 + \tilde{f}(w) + \int_{\Reals^d} \tilde{f}(\tilde{w})\mu(\dif \tilde{w}) + \int_{\Reals^d} K(w,\tilde{w})\mu(\dif \tilde{w}) \\
		&= R_0 + \int_{\Reals^d} \tilde{f}(\tilde{w})\mu(\dif \tilde{w}) + \Psi(w;\mu).
	\end{align*}
Indeed, $R(\mu) = \E_\mu[\bar{c}(W;\mu)]$ for any $\mu \in \fP(\Reals^d)$. Thus, the problem of minimizing $J_{\beta,T}(\bd{u})$ is an instance of controlled McKean--Vlasov dynamics \citep{carmona2015MKV} with running cost $c(x,u) = \frac{1}{2}\|u\|^2$ and terminal cost $\frac{\beta}{2}\bar{c}(x,\mu)$.
	
We now follow the formulation in Section~6.1 of \cite{carmona2018vol1} to solve this McKean--Vlasov control problem.\footnote{In the general McKean--Vlasov framework, the drift, the control cost, and the terminal cost may depend on the state, the control, and the marginal probability law of the state. Here, however, only the terminal cost has this dependence, which simplifies things considerably.} Without loss of generality, we can restrict the optimization to Markovian controls of the form $\bd{u} = \{\varphi(W_t,t)\}_{t \in [0,1]}$ for some deterministic function $\varphi : \Reals^d \times [0,1] \to \Reals^d$; for $\bd{u}$ of this form, we will write $\mu^\varphi_t$ instead of $\mu^{\bd{u}}_t$. We seek a pair $(\bd{\rho}^\star,V^\star)$, where $\bd{\rho}^\star = (\rho^\star_t)_{t \in [0,1]}$ is a flow of probability densities on $\Reals^d$ and $V^\star$ is a real-valued $C^{2,1}(\Reals^d \times [0,1])$ function that jointly solve the forward-backward system
\begin{subequations}
	\begin{align}
		\partial_t \rho^\star_t(w) &=  \nabla_w \cdot (\rho^\star_t(w)\nabla_w V^\star(w,t)) + \frac{1}{2}\Delta_w \rho^\star_t(w) \label{eq:FPE}\\
		\partial_t V^\star(w,t) &= - \frac{1}{2}\Delta_w V^\star(w,t) + \frac{1}{2}\|\nabla_w V^\star(w,t)\|^2 \label{eq:HJB}
	\end{align}
\end{subequations}
	on $\Reals^d \times [0,1]$, where \eqref{eq:FPE} has initial condition $\rho^\star_0(w) = \delta(w)$, \eqref{eq:HJB} has terminal condition
	\begin{align}\label{eq:HJB_terminal}
		V^\star(w,1) = \frac{\beta}{2} \left[\bar{c}(w,\mu_1) + \int_{\Reals^d} \frac{\delta \bar{c}}{\delta \mu}(\tilde{w},\mu_1)(w)\mu^\star_1(\dif \tilde{w})\right]
	\end{align}
for $\mu^\star_t(\dif w) \deq \rho^\star_t(w)\dif w$, and where $\frac{\delta \bar{c}}{\delta \mu}(\tilde{w},\cdot)(\cdot)$ is the linear functional derivative of $\mu \mapsto \bar{c}(\tilde{w},\mu)$  (with $\tilde{w}$ fixed), cf.~\citet[Section~5.4.1]{carmona2018vol1} or Appendix~\ref{app:functional_derivatives}. If such a pair $(\bd{\rho}^\star,V^\star)$ is found, then the optimal control is given by $\varphi(w,t) = -\nabla_w V^\star(w,t)$ and $\mu^\varphi_t(\dif w) = \rho^\star_t(w)\dif w$ for all $t$.
	
	In this instance, since $\mu \mapsto \bar{c}(\tilde{w},\mu)$ is linear, we can take 
	\begin{align*}
		\frac{\delta \bar{c}}{\delta \mu}(\tilde{w},\nu)(w) = R_0 + \tilde{f}(\tilde{w}) + \tilde{f}(w) + K(w,\tilde{w})
	\end{align*}
	and thus the terminal condition \eqref{eq:HJB_terminal} becomes
	\begin{align*}
		V^\star(w,T) &= \frac{\beta}{2}\left[\bar{c}(w,\mu_T) + R_0 + \tilde{f}(w) + \int_{\Reals^d}\tilde{f}(\tilde{w})\mu_T(\dif\tilde{w}) + \int_{\Reals^d} K(w,\tilde{w})\mu_T(\dif\tilde{w}) \right] \\
		&= \beta \bar{c}(w,\mu_T).
	\end{align*}
	
	Now let $\mu^\star$ be the minimizer of $\FE_{\beta,T}(\cdot)$, and consider the Cauchy problem
	\begin{align*}
		\partial_t \hat{V}(w,t) = - \frac{1}{2}\Delta_w \hat{V}(w,t) + \frac{1}{2}\|\nabla_w \hat{V}(w,t)\|^2, \qquad (w,t) \in \Reals^d \times [0,T]
	\end{align*}
	with the terminal condition $\hat{V}(w,T) = \beta \bar{c}(w,\mu^\star)$. Then, making the logarithmic (or Cole-Hopf) transformation $\hat{h}(w,t) \deq e^{-\hat{V}(w,t)}$ \citep{fleming1978exit,fleming1985fundamental}, it is not hard to verify that $\hat{h}$ solves the Cauchy problem
\begin{align}\label{eq:backward_h}
	\partial_t \hat{h}(w,t) = - \frac{1}{2}\Delta_w \hat{h}(w,t), \qquad (w,t) \in \Reals^d \times [0,T]
\end{align}
with terminal condition
\begin{align*}
	\hat{h}(w,T) &= \exp\left(-\beta \bar{c}(w,\mu^\star)\right) \\
	&= \exp\left(-M(\beta,T;\mu^\star)\right) \exp\left( -\beta \Psi(w;\mu^\star)\right),
\end{align*}
where $M(\beta,T;\mu^\star)$ does not depend on $w$. The solution to \eqref{eq:backward_h} is given by the Feynman--Kac formula \citep{kallenberg2002prob_book}:
\begin{align*}
	\hat{h}(w,t) &=  \exp\left(-M(\beta,T;\mu^\star)\right)\E\Bigg[ \exp\left( - \beta \Psi(B_T;\mu^\star)\right)\Bigg|B_t = w\Bigg],
\end{align*}
or
\begin{align*}
	\hat{V}(w,t) &= M(\beta,T;\mu^\star) - \log \E\Bigg[ \exp\left( - \beta \Psi(B_T;\mu^\star)\right)\Bigg|B_t = w\Bigg] \\
	&= M(\beta,T;\mu^\star) - \log \E\Bigg[ \exp\left( - \beta \Psi(B_T;\mu^\star)\right)\Bigg|B_{t} = w\Bigg].
\end{align*}
Consider now the SDE
\begin{align*}
	\dif W_t = - \nabla_w \hat{V}(W_t,t)\dif t + \dif B_t, \qquad t \in [0,T]; \, \, W_0 = 0.
\end{align*}
Then we know from the results of \citet{daipra1991reciprocal}  that the marginal distribution $\hat{\mu}_t$ of $W_t$ is given by $\hat{\mu}_t(\dif w) = \hat{\rho}_t(w)\dif w$ with density
\begin{align}\label{eq:barp_t}
	\hat{\rho}_t(w) &= \frac{1}{(2\pi t)^{d/2}}\exp\left(-\frac{\|w\|^2}{2 t}\right) \frac{\hat{h}(w,t)}{\hat{h}(0,0)}.
\end{align}
In particular,
\begin{align*}
	\hat{\rho}_T(w) = \frac{1}{(2\pi T)^{d/2}} \frac{\exp\left(-\frac{\|w\|^2}{2T}-\beta \Psi(w;\mu^\star)\right)}{Z(\beta,T;\mu^\star)},
\end{align*}
so $\hat{\mu}_T = \mu^\star$, since $\mu^\star$ satisfies the Boltzmann fixed-point condition \eqref{eq:Boltzmann_FP}. Thus, $\hat{V}$ solves the backward equation
\begin{align*}
	\partial_t \hat{V}(w,t) = - \frac{1}{2}\Delta_w \hat{V}(w,t) + \frac{1}{2}\|\nabla_w \hat{V}(w,t)\|^2, \qquad (w,t) \in \Reals^d \times [0,T]
\end{align*}
with the terminal condition $\hat{V}(w,T) = \beta \bar{c}(w,\mu^\star) = \beta \bar{c}(w,\hat{\mu}_T)$. It is also straightforward to show that the flow of densities $\hat{\bd{\rho}} = (\hat{\rho}_t)_{t \in [0,T]}$ solves the forward equation
\begin{align*}
	\partial_t \hat{\rho}_t(w) = \nabla_w \cdot \left(\hat{\rho}_t(w)\nabla_w \hat{V}(w,t)\right) + \frac{1}{2}\Delta_w \hat{\rho}_t(w),\qquad (w,t) \in \Reals^d \times [0,T]
\end{align*}
with the initial condition $\hat{\rho}_0(w) =\delta(w)$. (It is easy to show this directly by differentiating both sides of \eqref{eq:barp_t} with respect to time and using the fact that $w \mapsto \frac{1}{(2\pi t)^{d/2}}e^{-\|w\|^2/2 t}$ is the fundamental solution of the heat equation $\partial_t \phi = \frac{1}{2}\Delta \phi$ and $\hat{h}$ solves the backward problem \eqref{eq:backward_h}.) Hence, $(\bd{\rho^\star},V^\star) = (\hat{\bd{\rho}},\hat{V})$ is the pair we seek that solves the McKean--Vlasov forward-backward system. We have thus proved that \eqref{eq:optimal_SDE} and \eqref{eq:value_function} specify the optimal controlled dynamics, and that $W_T \sim \mu^\star$.

Finally, by Theorem~\ref{thm:Follmer} in Appendix~\ref{app:Follmer}, the drift of the process \eqref{eq:optimal_SDE} is the F\"ollmer drift for $\mu^\star$. As a consequence,
	\begin{align*}
		\E\left[\frac{1}{2}\int^T_0 \|\varphi(W_t,t)\|^2\dif t\right]  = D(\mu^\star \|\gamma_T),
	\end{align*}
	and hence
	\begin{align*}
		\min_{\bd{u}} J_{\beta,T}(\bd{u}) =  D(\mu^\star\|\gamma_T) + \frac{\beta}{2}R(\mu^\star) = \beta \FE_{\beta,T}(\mu^\star).
	\end{align*}
\end{proof}

Note that the control law in Eqs.~\eqref{eq:optimal_SDE}--\eqref{eq:value_function} is also optimal in the following sense \citep{daipra1991reciprocal,lehec2013entropy,eldan2018diffusion,eldan2018gausswidth}: Let $\fU(\mu^\star)$ denote the subset of all admissible drifts that obey the terminal condition $W^{\bd{u}}_T \sim \mu^\star$. Then
	\begin{align}\label{eq:Schrodinger}
		\inf_{\bd{u} \in \fU(\mu^\star)}\E\left[\frac{1}{2}\int^T_0 \|u_t\|^2 \dif t \right] = D(\mu^\star \| \gamma_T),
	\end{align}
	and the infimum is achieved by the drift in \eqref{eq:optimal_SDE}. Thus the optimization problem in \eqref{eq:Schrodinger} is indeed the Schr\"odinger bridge problem, with the optimal drift being the {\em F\"ollmer drift} \citep{follmer1985reversal}. 

It is instructive to take a closer look at the structure of the F\"ollmer drift in \eqref{eq:optimal_SDE}. To start, a simple computation gives
\begin{align}
	-\nabla_w V^\star(w,t) = -\frac{\beta \int_{\Reals^d} \nabla \Psi(v;\mu^\star)\exp\left(-\frac{\|v-w\|^2}{2(T-t)}-\beta\Psi(v;\mu^\star)\right)\dif v}
	{\int_{\Reals^d}\exp\left(-\frac{\|v-w\|^2}{2(T-t)}-\beta\Psi(v;\mu^\star)\right)\dif v}.\label{eq:Follmer_explicit}
\end{align}
If we define a family of measures $\{ Q_{w,t} : w \in \Reals^d, t \in [0,T]\} \subseteq \fP(\Reals^d)$ by
\begin{align}\label{eq:Q_measures}
	Q_{w,t}(A) &\deq \frac{\int_{A} \exp\left(-\frac{\|v-w\|^2}{2(T-t)}-\beta\Psi(v; \mu^\star)\right)\dif v}{\int_{\Reals^d} \exp\left(-\frac{\|v-w\|^2}{2(T-t)}-\beta\Psi(v; \mu^\star)\right)\dif v},
\end{align}
where $A$ ranges over all Borel subsets of $\Reals^d$, then we can write \eqref{eq:Follmer_explicit} more succinctly as
\begin{align}\label{eq:Follmer_Gibbs}
	-\nabla_w V^\star(w,t) = -\beta \int_{\Reals^d}\nabla \Psi(v;\mu^\star) Q_{w,t}(\dif v),
\end{align}
 An inspection of \eqref{eq:Q_measures} reveals that the flow $t \mapsto Q_{0,t}$ interpolates between $\mu^\star$ at $t=0$ and $\delta_0$ at $t=T$, so the measures $Q_{0,t}$ get increasingly concentrated as $t$ approaches $T$, cf.~Lemma~\ref{lm:Gstar}. Moreover, it can be shown that the flow of random measures $\{Q_{W_t,t}\}_{t \in [0,T]}$ along the trajectory of \eqref{eq:optimal_SDE} satisfies
\begin{align*}
	Q_{W_t,t}(\cdot) = \P[W_T \in \cdot|\cF_t] = \P[W_T \in \cdot|W_t]
\end{align*}
almost surely \citep[Lemma~11]{eldan2018gausswidth}. Using these facts, one can readily verify that the drift in \eqref{eq:optimal_SDE} can be written as
\begin{align}\label{eq:optimal_drift}
 -\beta \int_{\Reals^d} \nabla \Psi(v; \mu^\star) Q_{W_t,t}(\dif v) \equiv - \beta \cdot \E[\nabla \Psi(W_T;\mu^\star)|W_t],
\end{align}
i.e., it is equal to the conditional mean of the scaled negative gradient $-\beta \nabla \Psi(W_T;\mu^\star)$ given $W_t$. Note, however, that the potential $\Psi(W_T; \mu^\star)$ is a function of both $W_T$ and of the marginal distribution of $W_T$. This provides a nice illustration of the ``nonlocal'' nature of optimal control laws in control problems of McKean--Vlasov type \citep{carmona2018vol1}.

\section{Mean-field Langevin diffusion as approximation to optimal control}

In addition to being computationally expensive to represent (implement), the optimal control requires \emph{a priori} knowledge of the terminal distribution $\mu^\star$, which by nature is not available in the context of an optimization problem.  It is thus useful to examine how well various suboptimal processes can approximate the behavior of the F\"ollmer drift over a finite time interval.  As it is simple to implement and a well-researched continuous-time analogue to optimization algorithms used in practice (gradient descent), the Langevin diffusion is a natural choice to evaluate.  This is akin to, from an algorithmic standpoint, evaluating the goodness of a greedy algorithm relative to an optimal dynamic programming approach.  In the continuous setting here, this will reflect the strength of entropic regularization relative to the length of optimization; the quality of approximation will depend on the extent to which the Langevin diffusion behaves like an Ornstein-Uhlenbeck process.

In particular, we are interested in comparing the process $(W_t)_{t \in [0,T]}$ given in \eqref{eq:optimal_SDE} to the following two processes: the \textit{ideal} Langevin dynamics
\begin{align}\label{eq:MKV_Langevin_1}
	\dif \tilde{W}_t &=  -\frac{1}{2}\Big(\beta\nabla\Psi(\tilde{W}_t;\mu^\star) + \frac{\tilde{W}_t}{T}\Big) \dif t + \dif B_t;
\end{align}
which corresponds to the overdamped Langevin diffusion using the optimal (Boltzmann fixed-point) law; and the \textit{mean-field} Langevin dynamics
\begin{align}\label{eq:MKV_Langevin_2}
\dif \hat{W}_t = -\frac{1}{2}\Big(\beta\nabla \Psi(\hat{W}_t;\mu_t) + \frac{\hat{W}_t}{T}\Big) \dif t + \dif B_t,
\end{align}
which corresponds to the Langevin diffusion using its own law $\mu_t = \Law(\hat{W}_t)$ at time $t$, making it more directly comparable to an implementable method. The drift in \eqref{eq:MKV_Langevin_1} is chosen so that $\mu^\star_{\beta,T}$ is the (unique) invariant distribution. 

\subsection{Optimal control vs.\ the ideal Langevin dynamics}

The main difference between the process \eqref{eq:optimal_SDE} and the ideal Langevin dynamics \eqref{eq:MKV_Langevin_1} is that the former is guaranteed to produce a sample from $\mu^\star$ at time $T$, while for the latter we can only guarantee the convergence of $\Law(\tilde{W}_t)$ to $\mu^\star$ as $t \to \infty$. Nevertheless, we can show that $\tilde{W}_T$ will be close to $W_T$ provided both $\beta T$ and $T$ are sufficiently small.

\begin{theorem} Consider a synchronous coupling of the optimally controlled process \eqref{eq:optimal_SDE} with the ideal Langevin dynamics \eqref{eq:MKV_Langevin_1} using the same Brownian motion and the same initial condition $W_0 = \tilde{W}_0 = 0$. Then the following holds with probability at least $1-\delta$:
	\begin{align}
		\|W_T - \tilde{W}_T\|^2 \le \frac{2c_3}{\lambda^3} + \frac{c_2}{\lambda^2} + \frac{c_1}{\lambda},
	\end{align}
	where
	\begin{align}\label{eq:MKV_vs_Langevin_constants}
		\begin{split}
		\lambda &\deq \frac{1}{2T} - \kappa\beta^2 T, \\
		c_3 &\deq \kappa\beta^4 T, \\
		c_2 &\deq \kappa\beta^2 Td, \\
		c_1 &\deq \kappa\left(\beta^2 T+d\log \frac{2d}{\delta}\right),
		\end{split}
	\end{align}
	provided $\beta$ and $T$ are such that $\lambda > 0$.
\end{theorem}
\begin{proof} Taking into account the representation \eqref{eq:Follmer_Gibbs} of the drift in \eqref{eq:optimal_SDE}, we can write
	\begin{align*}
		& \frac{\dif}{\dif t}\|\tilde{W}_t - W_t\|^2 \nonumber\\
		&= -2\left\langle\tilde{W}_t - W_t,  \frac{\beta}{2}\nabla \Psi(\tilde{W}_t;\mu^\star) + \frac{\tilde{W}_t}{2T} - \beta\int_{\Reals^d} \nabla \Psi(v;\mu^\star)Q_{t,W_t}(\dif v)\right\rangle \\
		&= - \frac{1}{T}\|\tilde{W}_t - W_t\|^2 + 2\int_{\Reals^d} \left\langle\tilde{W}_t - W_t,  \beta \nabla \Psi(v;\mu^\star) - \frac{\beta}{2}\nabla \Psi(\tilde{W}_t;\mu^\star) - \frac{W_t}{2T} \right\rangle Q_{t,W_t}(\dif v) \\
		&\le - \frac{1}{2T}\|\tilde{W}_t - W_t\|^2 +  2T\int_{\Reals^d}  \left\|  \beta \nabla \Psi(v;\mu^\star)- \frac{W_t}{2T}  - \frac{\beta}{2}\nabla \Psi(\tilde{W}_t;\mu^\star) \right\|^2 Q_{t,W_t}(\dif v),
	\end{align*}
	where we have used the inequality $2\ave{v,w} \le \frac{1}{2T}\|v\|^2 + 2T\|w\|^2$. 
	
Using the properties of $\Psi$, we can estimate
	\begin{align*}
		 \left\|  \beta \nabla \Psi(v;\mu^\star)- \frac{W_t}{2T}  - \frac{\beta}{2}\nabla \Psi(\tilde{W}_t;\mu^\star) \right\|^2 \le \kappa \left(\beta^2 \|W_t - \tilde{W}_t \|^2 + \beta^2 \|W_t - v \|^2 + \frac{1}{T^2}\sup_{t \in [0,T]}\|W_t\|^2\right),
	\end{align*}
where, by the boundedness of $\nabla \Psi$, we can further estimate
	\begin{align*}
		\sup_{0 \le t \le T} \|W_t\| &\le \kappa \beta T +  \sup_{0 \le t \le T} \|B_t\| \\
		&\stackrel{{\rm{d}}}{=} \kappa \beta T +  \sqrt{T}\sup_{0 \le t \le 1} \|B_t\|,
	\end{align*}
	where $\stackrel{{\rm{d}}}{=}$ denotes equality in distribution. By the reflection principle for the Brownian motion,
\begin{align*}
	\P\left\{ \sup_{0 \le t \le 1} \|B_t\| \ge \alpha \right\}  \le 2 d e^{-\frac{\alpha^2}{2d}},
\end{align*}
so that, with probability at least $1-\delta$,
\begin{align*}
	\sup_{0 \le t \le T} \|W_t\|^2 \le \left(\kappa\beta T + \sqrt{2T d\log \frac{2d}{\delta}}\right)^2 \le \kappa\left(\beta^2 T^2 + T d\log \frac{2d}{\delta}\right).
\end{align*}
Using the above estimates together with Lemma~\ref{lm:Gstar}, we deduce that the following differential inequality holds for all $\Delta_t \deq \|W_t  - \tilde{W}_t\|^2$, $0 \le t \le T$, with probability at least $1-\delta$:
\begin{align*}
	\frac{\dif}{\dif t}\Delta_t \le -\lambda \Delta_t + c_3(T-t)^2 + c_2(T-t) + c_1, \qquad 0 \le t \le T
\end{align*}
where $\lambda$ and $c_{1,2,3}$ are defined in \eqref{eq:MKV_vs_Langevin_constants}. Integrating from $t=0$ to $t=T$ and using the fact that $\Delta_0 = 0$, we see that, with probability at least $1-\delta$,
\begin{align*}
	\Delta_T \le \frac{2c_3}{\lambda^3} + \frac{c_2}{\lambda^2} + \frac{c_1}{\lambda}.
\end{align*}
This completes the proof.
\end{proof}

By way of illustration, take $\beta = \eps^{-1/2}$ and $T = \eps$, where $\eps \in (0,1)$ is chosen so that $\lambda > 0$ in \eqref{eq:MKV_vs_Langevin_constants}. Then we have
\begin{align*}
	\lambda \asymp \eps^{-1}, \quad c_3 \asymp \eps^{-1}, \quad c_2 \asymp d, \quad c_1 \asymp d \log \frac{2d}{\delta},
\end{align*}
and in that regime the error bound
\begin{align*}
	\| W_\eps - \tilde{W}_\eps \|^2 \lesssim \eps d \log \frac{2d}{\delta}
\end{align*}
holds with probability at least $1-\delta$. With a standard argument this can be converted to the expected $L^2$ error bound,
\begin{align*}
	\E \|W_\eps - \tilde{W}_\eps \|^2 \lesssim \eps d^2.
\end{align*}
Denoting $\Law(\tilde{W}_t)$ by $\tilde{\mu}_t$, we obtain the following $L^2$ Wasserstein estimate relating the optimal control to the ideal Langevin dynamics in short-time ($T \sim \eps$) and low-temperature ($\beta^{-1} \sim \eps^{1/2}$) regime:
\begin{align*}
	\Wass^2_2(\tilde{\mu}_\eps,\mu^\star_{\eps^{-1/2},\eps}) \lesssim  \eps d^2.
\end{align*}

\subsection{Ideal vs.\ mean-field Langevin dynamics}

\subsubsection{A lemma on exponential convergence}

A key ingredient of our analysis is a simple lemma pertaining to McKean--Vlasov diffusions governed by It\^o SDEs of the form
\begin{align}
	\dif W_t = -\nabla U(W_t,P_t)\dif t + \dif B_t, \qquad t \ge 0
\end{align}
where $P_t \deq \Law(W_t)$. In particular, we are interested in analyzing the convergence of $P_t$ to a unique invariant distribution. Note that the mean-field Langevin dynamics \eqref{eq:MKV_Langevin_2} is of this form, with
\begin{align}\label{eq:U_Langevin}
U(w,\mu) = \frac{\beta}{2} \Psi(w;\mu) + \frac{\|w\|^2}{4T}.
\end{align}

\begin{assumption}\label{as:potential}  The McKean--Vlasov potential $U : \Reals^d \times \fP(\Reals^d) \to \Reals$ obeys the following:
	\begin{itemize}
		\item strong convexity in $w$ -- the function $w \mapsto U(w,\mu)$ is differentiable for each $\mu \in \fP(\Reals^d)$, and there exists a constant $m > 0$, such that
		\begin{align}\label{eq:U_sc}
			\ave{w-w', \nabla U(w,\mu)-\nabla U(w',\mu)} \ge m \| w - w' \|^2
		\end{align}
		for all $w,w' \in \Reals^d$ and all $\mu \in \fP(\Reals^d)$;
		\item smoothness in $\mu$ -- there exists a constant $L \ge 0$, such that
		\begin{align}\label{eq:U_smooth}
			\| \nabla U(w,\mu) - \nabla U(w,\nu) \| \le L\Wass_1(\mu,\nu)
		\end{align}
		for all $w \in \Reals^d$ and all $\mu,\nu \in \fP(\Reals^d)$;
		\item Boltzmann fixed-point condition -- there exists a unique $\bar{\mu} \in \fP(\Reals^d)$, such that
		\begin{align}\label{eq:U_Boltzmann}
			\bar{\mu}(\dif w) = \frac{1}{Z} \exp\left(-2 U(w,\bar{\mu})\right) \dif w,
		\end{align}
		where $Z$ is the normalization constant.
	\end{itemize}
\end{assumption}

\begin{lemma}\label{lm:expcon} Suppose that the McKean--Vlasov potential $U$ satisfies Assumption~\ref{as:potential}, and moreover that $m > L$. Then
	\begin{align}
		\Wass_2(P_t,\bar{\mu}) \le \Wass_2(P_0,\bar{\mu}) e^{-(m-L)t}.
	\end{align}
\end{lemma}
\begin{proof} Let $(\mu_t)_{t \ge 0}$ and $(\tilde{\mu}_t)_{t \ge 0}$ be two arbitrary (possibly  exogenous) flows of probability measures on $\Reals^d$, and consider two SDEs, coupled synchronously through a common Brownian motion process:
	\begin{align*}
		\dif W_t &= -\nabla U(W_t,\mu_t) \dif t + \dif B_t, \\
		\dif \tilde{W}_t &= -\nabla U(\tilde{W}_t,\tilde{\mu}_t) \dif t + \dif B_t
	\end{align*}
	for $t \ge 0$. Let $P_t \deq {\rm Law}(W_t)$ and $\tilde{P}_t \deq {\rm Law}(\tilde{W}_t)$. Then, using \eqref{eq:U_sc} and \eqref{eq:U_smooth}, we obtain
	\begin{align}
		\frac{\dif}{\dif t}\|W_t - \tilde{W}_t \|^2 &= 2 \ave{W_t - \tilde{W}_t, \nabla U(\tilde{W}_t,\tilde{\mu}_t) - \nabla U(W_t,\mu_t)} \nonumber\\
		&\le 2 \ave{W_t - \tilde{W}_t, \nabla U(\tilde{W}_t,\tilde{\mu}_t)-\nabla U(W_t, \tilde{\mu}_t)} + 2 \ave{W_t - \tilde{W}_t, \nabla U(W_t,\tilde{\mu}_t)-\nabla U(W_t, \mu_t)} \nonumber\\
		&\le -2m \|W_t - \tilde{W}_t \|^2 + 2L \|W_t - \tilde{W}_t \| \Wass_1(\mu_t,\tilde{\mu}_t).\label{eq:sync_couple_1}
	\end{align}
Now, the second term on the right-hand side of \eqref{eq:sync_couple_1} can be further upper-bounded by
\begin{align*}
	2L \|W_t - \tilde{W}_t \| \Wass_1(\mu_t,\tilde{\mu}_t) \le L \left( \|W_t - \tilde{W}_t \|^2 + \Wass^2_1(\mu_t,\tilde{\mu}_t)\right),
\end{align*}
which gives
\begin{align*}
	\frac{\dif}{\dif t}\|W_t - \tilde{W}_t \|^2  &\le - (2m - L) \|W_t - \tilde{W}_t \|^2 + L\Wass^2_1(\mu_t,\tilde{\mu}_t).
\end{align*}
Integrating, we obtain
\begin{align*}
	\| W_t - \tilde{W}_t \|^2 &\le e^{-(2m-L)t} \|W_0 - \tilde{W}_0 \|^2 + L\int^t_0 e^{(2m-L)(s-t)}\Wass^2_1(\mu_s,\tilde{\mu}_s) \dif s \\
	&\le e^{-(2m-L)t} \|W_0 - \tilde{W}_0 \|^2 +L\int^t_0 e^{(2m-L)(s-t)}\Wass^2_2(\mu_s,\tilde{\mu}_s) \dif s
\end{align*}
Assume now that the joint law of $(W_0,\tilde{W}_0)$ achieves $\Wass^2_2(P_0,\tilde{P}_0)$, so that
\begin{align*}
	\Wass^2_2(P_t,\tilde{P}_t) \le e^{-(2m-L)t} \Wass^2_2(P_0,\tilde{P}_0) + L\int^t_0 e^{(2m-L)(s-t)}\Wass^2_2(\mu_s,\tilde{\mu}_s) \dif s.
\end{align*}
Now we apply the above to the case when $\mu_t = P_t$ and $\tilde{\mu}_t = \bar{\mu}$ for $t \ge 0$. Then the Boltzmann fixed-point condition \eqref{eq:U_Boltzmann} implies that $\tilde{P}_t = \bar{\mu}$, and
\begin{align*}
	\Wass^2_2(P_t,\bar{\mu}) &= \Wass^2_2(P_t,\tilde{P}_t) \\
	&\le  e^{-(2m-L)t} \Wass^2_2(P_0,\tilde{P}_0) +L\int^t_0 e^{(2m-L)(s-t)}\Wass^2_2(P_s,\tilde{P}_s) \dif s.
\end{align*}
For $u(t) \deq e^{(2m-L)t}\Wass^2_2(P_t,\tilde{P}_t)$, we have
\begin{align*}
	u(t) \le \Wass^2_2(P_0,\tilde{P}_0) + L\int^t_0 u(s) \dif s,
\end{align*}
so Gr\"onwall's lemma gives $u(t) \le \Wass^2_2(P_0,\bar{\mu})e^{Lt}$.
\end{proof}

\subsubsection{Strictly log-concave density}

It is straightforward to verify Assumption~\ref{as:potential} for the potential $U$ in \eqref{eq:U_Langevin} when the density $\mu^\star_{\beta,T}$ is strictly log-concave, i.e., when
$$
\nabla^2 \Bigg[\beta \Psi(\cdot;\mu^\star) + \frac{1}{2T}\|\cdot\|^2\Bigg]\succ 0.
$$
This will be guaranteed when $\frac{1}{T}-\beta\kappa > 0$, where $\kappa = \sigma_{\max}\big(\nabla^2 \Psi(\cdot,\mu^\star)\big)>0$. In this case, the ideal Langevin diffusion is known to have quick convergence to the stationary Gibbs distribution (i.e., in time polynomial in all relevant parameters).  In particular, by the Bakry--{\'E}mery criterion \citep{Bakry_Gentil_Ledoux_book}, the log-Sobolev constant of $\mu^\star$ is given by
$$
c_\text{LS} = O\Big(\frac{1}{1/T-\beta\kappa}\Big),
$$
so that the choice $\beta T  = \frac{1}{c'\kappa}$ gives $c_\text{LS} = \frac{c'}{c'-1}T$, where $\frac{c'}{c'-1} = O(1)$ (take, e.g., $c'=2$). This implies convergence of $\tilde{\mu}_t = \Law(\tilde{W}_t)$ to the Gibbs distribution $\mu^\star$ at the exponential rate $e^{-\frac{t}{O(T)}}$.

From the above, Assumption~\ref{as:potential} holds with $m \gtrsim \frac{1}{T} - \beta\kappa$ and $L \lesssim \beta \kappa$. Hence, with $\mu_t$ denoting $\Law(W_t)$ in \eqref{eq:MKV_Langevin_2}, Lemma~\ref{lm:expcon} gives
$$
\Wass_2^2(\mu_t,\mu^\star) \leq \eps, \qquad \text{for } t = \Omega\Big(T\log\frac{1}{\eps}\Big).
$$

\subsubsection{The general case}

When strict log-concavity cannot be guaranteed, we can still bound the log-Sobolev constant of $\mu^\star$ using a perturbation argument of \cite{Holley1987}: Since
\begin{align*}
	\mu^\star(\dif w) \propto \exp\left(-\beta\Psi(w;\mu^\star)\right) \gamma_T(\dif w)
\end{align*}
and since $\Psi$ is bounded, the log-Sobolev constant of $\mu^\star$ can be estimated by
\begin{align}
	c_{\text{LS}}(\mu^\star) \le e^{2\beta \|\Psi\|_\infty} c_{\text{LS}}(\gamma_T)  \le Te^{\kappa \beta}.
\end{align}
By the results of \cite{chizat22MKV} and \cite{nitanda2022MKV}, the law $\mu_t$ of $\hat{W}_t$ in the mean-field Langevin dynamics \eqref{eq:MKV_Langevin_2} will be $\eps$-close to $\mu^\star$ in squared $L^2$ Wasserstein distance when $t \gtrsim c_{\text{LS}}(\mu^\star)\log \frac{1}{\eps} \sim Te^{\beta}\log \frac{1}{\eps}$.

\begin{appendix}
	\setcounter{equation}{0}
	\renewcommand{\theequation}{A.\arabic{equation}}
	\setcounter{theorem}{0}
	\renewcommand{\thetheorem}{A.\arabic{theorem}}
	\setcounter{definition}{0}
	\renewcommand{\thedefinition}{A.\arabic{definition}}
	\setcounter{lemma}{0}
	\renewcommand{\thelemma}{A.\arabic{lemma}}

\section{Miscellanea}

\subsection{Inequalities for the Gaussian measure}
\label{sec:inequalities}

Let $\mu$ be a Borel probability measure on $\Reals^d$, which is absolutely continuous w.r.t.\ the Gaussian measure $\gamma_T$: $\dif \mu = e^F \dif \gamma_T$ for some differentiable function $F : \Reals^d \to \Reals$. The following inequalities relate the relative entropy $D(\mu \| \gamma_T) = \int_{\Reals^d} F \dif \mu$, the Fisher information distance $I(\mu \| \gamma_T) = \int_{\Reals^d} \| \nabla F\|^2 \dif\mu$, and the $L^2$ Wasserstein distance $\Wass_2(\mu,\gamma_T)$ (see \citet{Bakry_Gentil_Ledoux_book} for a detailed treatment):
\begin{itemize}
	\item \textbf{log-Sobolev inequality} ---
	\begin{align}\label{eq:LSI}
		D(\mu \| \gamma_T) \le \frac{T}{2} I(\mu \| \gamma_T)
	\end{align}
	\item \textbf{entropy-transport inequality} ---
	\begin{align}\label{eq:Talagrand}
		\Wass^2_2(\mu, \gamma_T) \le 2T D(\mu \| \gamma_T).
	\end{align}
\end{itemize}

\subsection{Functions of measures and their linear functional derivatives}
\label{app:functional_derivatives}

Recall that $\fP_2(\Reals^d)$ denotes the space of all Borel probability measures on $\Reals^d$ with finite second moment, equipped with the $L^2$ Wasserstein distance $\Wass_2(\cdot,\cdot)$. Let a function $G : \fP_2(\Reals^d) \to \Reals$ be given. The following definition can be found in Section~5.4 of \citet{carmona2018vol1}.

\begin{definition} We say that $G$ has a {\em linear functional derivative} if there exists a function
	\begin{align}
		(\nu,w) \mapsto \frac{\delta G}{\delta \mu}(\nu)(w)
	\end{align}
which is jointly continuous in $\nu$ and $w$, such that, for all $\nu,\nu' \in \fP_2(\Reals^d)$,
\begin{align}
	G(\nu') - G(\nu) = \int^1_0 \int_{\Reals^d} \frac{\delta G}{\delta \mu}(t\nu' + (1-t)\nu)(w) [\nu'(\dif x) - \nu(\dif x)] \dif t.
\end{align}
\end{definition}
For instance, it is easy to verify that if $G$ is linear in $\mu$, i.e., $G = \int_{\Reals^d} g\dif \mu$ for some function $g : \Reals^d \to \Reals$ of at most quadratic growth, then $\frac{\delta G}{\delta \mu}(\nu)(\cdot) = g(\cdot)$ for all $\nu$.

\subsection{The F\"ollmer drift}
\label{app:Follmer}

The following result, which can be found in different forms in the works of \citet{follmer1985reversal,daipra1991reciprocal,lehec2013entropy,eldan2018diffusion}, is used in the proof of Theorem~\ref{thm:dynamic}.

\begin{theorem}\label{thm:Follmer} Let a probability measure $\mu \in \fP(\Reals^d)$ be given, such that $\mu \ll \gamma_T$. Consider the It\^o SDE
	\begin{align}\label{eq:Follmer_SDE}
		\dif W_t = \varphi(W_t,t)\dif t + \dif B_t, \qquad W_0 = 0;\, 0 \le t \le T
	\end{align}
	where
	\begin{align}\label{eq:Follmer_drift}
		\varphi(w,t) = \nabla_w \log \E[\rho(B_T)|B_{t} = w], \qquad \rho \deq \frac{\dif \mu}{\dif \gamma_T}.
	\end{align}
Let $\fU(\mu)$ denote the collection of all progressively measurable drift processes $\bd{u}$ that obey
\begin{align}
	\E\Bigg[\int^T_0 \|u_t\|^2\Bigg] < \infty \qquad \text{and} \qquad \int^T_0 u_t \dif t + B_T \sim \mu.
\end{align}
Then the drift \eqref{eq:Follmer_drift} is an element of $\fU(\mu)$, and
\begin{align}
	\inf_{\bd{u} \in \fU(\mu)} \frac{1}{2} \E\left[\int^T_0 \|u_t\|^2 \dif t \right] = \frac{1}{2}\E\left[\int^T_0 \|\varphi(W_t,t)\|^2 \dif t \right] = D(\mu \| \gamma_T).
\end{align}	
\end{theorem}

\subsection{Technical lemmas}
\label{app:lemmas}

\begin{lemma}\label{lm:weak_compactness} Let $D > 0$ be given. Then the set $\fE(D) \deq \{\mu \in \fP_2(\Reals^d) : D(\mu\|\gamma_T) \le D\}$ is weakly compact.
\end{lemma}
\begin{proof} The function $\mu \mapsto D(\mu\|\gamma_T)$ is weakly lower-semicontinuous, so $\fE(D)$ is weakly closed. Now, by Talagrand's transportation inequality,
	\begin{align*}
		\mu \in \fE(D) \quad \Longrightarrow \quad \Wass^2_2(\mu,\gamma_T) \le 2 \tau D.
	\end{align*}
Let $\rho$ be the optimal coupling of $\mu$ and $\gamma_T$ that achieves $\Wass_2(\mu,\gamma_T)$. Then
	\begin{align*}
		\int_{\Reals^d}\|w\|^2 \mu(\dif w) &= \int_{\Reals^d \times \Reals^d} \|w\|^2 \rho(\dif w,\dif w') \\
		&\le 2\int_{\Reals^d \times \Reals^d} \|w-w'\|^2 \rho(\dif w,\dif w') + 2\tau d \\
		&= 2\Wass^2_2(\mu,\gamma_T) + 2\tau d \\
		&\le 2\tau(D+d).
	\end{align*}
	Thus, $\fE(D)$ is a closed subset of $\{ \mu \in \fP_2(\Reals^d) : \int_{\Reals^d}\|w\|^2\mu(\dif w) \le 2\tau(D+d)\}$, and the latter set is weakly compact. This establishes the compactness of $\fE(D)$.
\end{proof}

	\begin{lemma}[Maurey's empirical method --- high-probability version \citep{ji2020transport}]\label{lm:Maurey} Let a collection of functions $g(\cdot; w) \in L^2(\pi)$ be given, parametrized by $w \in \Reals^d$, and let $\mu \in \fP_2(\Reals^d)$. Let $W^1,\ldots,W^N$ be i.i.d.\ samples from $\mu$. Then, for $\hat{g} \deq \E_\mu[g(\cdot;W)]$, we have
		\begin{align*}
			\E \left\| \frac{1}{N}\sum^N_{i=1}g(\cdot; W^i) -  \hat{g}\right\|^2_{L^2(\pi)} \le \frac{1}{N}\sup_{w} \|g(\cdot;w)\|^2_{L^2(\pi)}
		\end{align*}
		and, moreover,
		\begin{align*}
			\left\| \frac{1}{N}\sum^N_{i=1}g(\cdot; W^i) - \hat{g}\right\|_{L^2(\pi)} \le \sup_{w} \| g(\cdot;w)\|_{L^2(\pi)} \left(\frac{1}{\sqrt{N}} + \sqrt{\frac{\log(1/\delta)}{N}}\right)
		\end{align*}
		with probability at least $1-\delta$.
	\end{lemma}
	
	\begin{lemma}\label{lm:Gstar} The probability measures $Q_{w,t}$ defined in \eqref{eq:Q_measures} satisfy
		\begin{align}\label{eq:Q_bound}
			\int_{\Reals^d} \| v - w \|^2 Q_{w,t}(\dif v) \le \kappa(T-t)(\beta^2(T-t)+d).
		\end{align}
	\end{lemma}
	\begin{proof} We first note that, by translation, we can consider instead the measure
		\begin{align*}
			\bar{Q}_{w,t}(\dif v) = \frac{1}{Z} \exp\left(-\beta \Psi(v+w)\right)\gamma_{T-t}(\dif v),
		\end{align*}
		where $Z$ is the normalizaton constant. Invoking Talagrand's entropy-transport inequality \eqref{eq:Talagrand} and the Gaussian log-Sobolev inequality \eqref{eq:LSI}, we can write
		\begin{align*}
			\Wass^2_2(\bar{Q}_{w,t}, \gamma_{T-t}) &\le 2(T-t) D(\bar{Q}_{w,t}\|\gamma_{T-t}) \\
			& \le \kappa \beta^2 (T-t)^2.
		\end{align*}
		Therefore, letting $\nu \in \fP_2(\Reals^d \times \Reals^d)$ be the optimal $L^2$ Wasserstein coupling of $\bar{Q}_{w,t}$ and $\gamma_{T-t}$, we can estimate
		\begin{align*}
			\int_{\Reals^d} \|v - w \|^2 Q_{w,t}(\dif v) &= \int_{\Reals^d} \|v\|^2 \bar{Q}_{w,t}(\dif v) \\
			&= \int_{\Reals^d} \|v\|^2 \nu(\dif v, \dif \tilde{v}) \\
			&\le 2\int_{\Reals^d} \|\tilde{v}\|^2 \gamma_{T-t}(\dif\tilde{v}) + 2\int_{\Reals^d} \| v - \tilde{v}\|^2 \nu(\dif v, \dif \tilde{v}) \\
			&= 2d (T-t) + 2\Wass^2_2(\bar{Q}_{w,t},\gamma_{T-t}) \\
			&\le 2 d (T-t) + 2\kappa \beta^2 (T-t)^2.
		\end{align*}
		This proves \eqref{eq:Q_bound}.
	\end{proof}
	
	\end{appendix}

\section*{Acknowledgments}

This work was supported by the NSF under awards CCF-2348624 (``Towards a control framework for neural generative modeling'') and CCF-2106358 (``Analysis and Geometry of Neural Dynamical Systems''), and by the Illinois Institute for Data Science and Dynamical Systems (iDS${}^2$), an NSF HDR TRIPODS institute, under award CCF-1934986.

\bibliography{mean_field.bbl}

\end{document}